\documentclass{article}

\usepackage{microtype}
\usepackage{graphicx}
\usepackage{subfigure}
\usepackage{booktabs} 

\usepackage{hyperref}

\usepackage[accepted]{icml2020}

\usepackage{silence}
\usepackage{macros}
\usepackage{lipsum}
\usepackage{mathabx}
\usepackage{cleveref}
\usepackage{thmtools}
\usepackage{longtable}
\usepackage{mathrsfs}
\usepackage{thm-restate}

\def\sqrtalphadef{\sqrt{\beta_{tk}} + \sqrt{k}\IBE + \sqrt{\lambda}\Radius_t}
\def\RecursionExpansion{\overline \xi_{t} + \mathring \theta_t(\Qbar_{t+1}) + \Sigma^{-1}_{tk}\sum_{i=1}^{k-1} \phi_{ti}^\top 
	\mathring \Delta_{t}(\Qbar_{t+1})(s_{ti},a_{ti}) -\lambda\Sigma^{-1}_{tk}\mathring \theta_t(\Qbar_{t+1}) + \Sigma^{-1}_{tk}\sum_{i=1}^{k-1} \phi_{ti} \eta_{ti}(\Vbar_{t+1}).
}

\def\Radius{\mathcal R}
\newcommand{\fullref}[1]{\cref{#1}}
\def\Alg{\ensuremath{\textsc{Eleanor}}}
\def\dotzetadef{\Big[ \E_{s' \sim p_t(s_{tk},a_{tk})}\(\Vbar_{t+1,k}-V^{\pi_k}_{t+1}\) (s') - \(\Vbar_{t+1,k}-V^{\pi_k}_{t+1}\) (s_{t+1,k}) \Big]\1\( \overline F_k \)}
\def\sqrtbetadef{\sqrt{d_t\ln\( 1 + \Lphi^2 k/d_t\) + 2d_{t+1}\ln(1+4\Radius_t\Lphi\sqrt{k}) + \ln\( \frac{1}{\delta'}\)} + 1}

\def\Lphi{L_\phi}

\def\thetastar{\theta^\star}

\def\IBE{\mathcal I}

\icmltitlerunning{Learning Near Optimal Policies with Low Inherent Bellman Error}

\begin{document}

\twocolumn[
\icmltitle{Learning Near Optimal Policies with Low Inherent Bellman Error}



\icmlsetsymbol{equal}{*}

\begin{icmlauthorlist}
\icmlauthor{Andrea Zanette}{SU}
\icmlauthor{Alessandro Lazaric}{FAIR}
\icmlauthor{Mykel Kochenderfer}{SU}
\icmlauthor{Emma Brunskill}{SU}
\end{icmlauthorlist}

\icmlaffiliation{SU}{Stanford University}
\icmlaffiliation{FAIR}{Facebook Artificial Intelligence Research}

\icmlcorrespondingauthor{Andrea Zanette}{zanette@stanford.edu}

\icmlkeywords{Machine Learning, ICML}

\vskip 0.3in
]



\printAffiliationsAndNotice{}

\begin{abstract}
We study the exploration problem with approximate linear action-value functions in episodic reinforcement learning under the notion of low inherent Bellman error, a condition normally employed to show convergence of approximate value iteration. First we relate this condition to other common frameworks and show that it is strictly more general than the low rank (or linear) MDP assumption of prior work. Second we provide an algorithm with a high probability regret bound $\widetilde O(\sum_{t=1}^H d_t  \sqrt{K} + \sum_{t=1}^H \sqrt{d_t} \IBE K)$ where $H$ is the horizon, $K$ is the number of episodes, $\IBE$ is the value if the inherent Bellman error and $d_t$ is the feature dimension at timestep $t$. In addition, we show that the result is unimprovable beyond constants and logs by showing a matching lower bound. This has two important consequences: 1) it shows that exploration is possible using only \emph{batch assumptions} with an algorithm that achieves the optimal statistical rate for the setting we consider, which is more general than prior work on low-rank MDPs 2) the lack of closedness (measured by the inherent Bellman error) is only amplified by $\sqrt{d_t}$ despite working in the online setting. Finally, the algorithm reduces to the celebrated \textsc{LinUCB} when $H=1$ but with a different choice of the exploration parameter that allows handling misspecified contextual linear bandits. While computational tractability questions remain open for the MDP setting, this enriches the class of MDPs with a linear representation for the action-value function where statistically efficient reinforcement learning is possible. 
\end{abstract}

\section{Introduction}

Improving the sample efficiency of reinforcement learning (RL) algorithms through effective exploration-exploitation strategies is a major focus of the recent theoretical literature. 
Strong results are available with a generative model \cite{Azar12,Sidford18,agarwal2019optimality,zanette2019b} as well as in the \emph{online} setting when the learning performance is measured by the cumulative regret, i.e., the difference between the performance of the optimal policy and the reward accumulated by the learner. For finite horizon problems, UCBVI~\cite{Azar17} achieves worst-case optimal regret, while algorithms with domain adaptive bounds have been introduced by~\cite{zanette2019tighter} and \cite{simchowitz2019non}. Randomized~\cite{russo2019worst} and model-free~\cite{jin2018q} variants have also been proposed, together with methods with other beneficial properties~\cite{dann2019policy,efroni2019tight}. Similar results are also available in the infinite horizon setting~\cite{Jaksch10,Maillard14,Fruit18,zhang2019regret,tossou2019near}.

\textbf{Approximate dynamic programming.}
While the results for tabular settings are encouraging, function approximation is normally required to tackle problems where the state or action spaces may be intractably large. In this case, even when the Bellman operator can be applied exactly, simple dynamic programming algorithms coupled with linear architectures may diverge~\cite{baird1995residual,tsitsiklis1996feature}, thus suggesting that effective approximate RL may not be feasible in the general case. 

Convergence guarantees \cite{lagoudakis2003least} and finite-sample analyses \cite{lazaric2012finite} are available for the least-squares policy improvement (LSPI) algorithm under the assumption that the value function of \emph{all policies can be well approximated} within the chosen function class (\emph{LSPI conditions}, for short). 
For concreteness, let $\epsilon$ be the worst-case misspecification error of a $d$-dimensional linear function approximator over the policy action-value functions (i.e., for any policy $\pi$, there exists an approximation $\widehat Q^\pi$ such that $\|\widehat Q^\pi - Q^\pi\|\leq \epsilon$). Recently, 
\cite{du2019good} showed that when using highly misspecified approximators $\epsilon \gtrapprox 1/\sqrt{d}$ the worst-case sample complexity may be exponential in $d$. At the same time, when $\epsilon \lessapprox 1/\sqrt{d}$, \cite{van2019comments} and \cite{lattimore2020learning} showed algorithms with $\sqrt{d}$ loss times the misspecification level $\epsilon$. In particular, \cite{lattimore2020learning} showed that LSPI attains polynomial sample complexity using $G$-optimal design with a $\approx \sqrt{d}\epsilon$ additive error using a \emph{generative model}.

Similarly, for the least-squares value iteration algorithm (LSVI) convergence guarantees \cite{munos2005error} and finite sample analysis \cite{munos2008finite} are also available under the assumption of \textit{low inherent Bellman error (IBE)}, (\emph{LSVI conditions}, for short). Given a function class $\mathcal F$, the IBE measures the error in approximating the image of any function in $\mathcal F$ through the Bellman operator. Whenever the IBE is not small, it is easy to show that approximation errors may be amplified by a constant factor at each application of the Bellman operator, leading to divergence. Although methods exist to limit this amplification of errors~\cite{zanette19limiting,kolter2011fixed}, the question of when sample-efficient value-based RL is possible remains open even in the absence of misspecification.

In this paper we focus on the problem of exploration-exploitation using LSVI approaches in settings with low IBE. We make several contributions.

\textbf{Exploration with low inherent Bellman error.}
We first show that the notion of inherent Bellman error is distinct from the LSPI condition, and more general than the low-rank assumption on the dynamics used in a series of recent works on exploration with linear function approximation~\cite{yang2020reinforcement,jin2020provably,zanette2020frequentist}.  
For a finite horizon MDP, when the LSVI conditions are satisfied either exactly or approximately (i.e., the inherent Bellman error is either zero or small) we propose \emph{Efficient Linear Exploration of Actions by Nonlinear Optimization of the Residuals} (\Alg{}), an optimistic generalization of the popular LSVI algorithm. We analyze \Alg{} and derive the first regret bound for this setting and show it is unimprovable in terms of statistical rates, though we leave its computational tractability open.

Our analysis shows that the performance of \Alg{} degrades gracefully in the case of positive inherent Bellman error. Interestingly, we recover a similar $\sqrt{d}$ amplification of the misspecification error (the IBE in our case) as for LSPI \cite{lattimore2020learning} , despite the fact that we consider the more challenging online setting as opposed to the generative model  by~\citet{lattimore2020learning}.  

\textbf{Low-rank MDPs and contextual misspecified linear bandits.} Our result applies to low-rank MDPs and improves upon the best-known regret bound for that setting \cite{jin2020provably} by a $\sqrt{d}$ factor. When applied to contextual linear bandits, our algorithm reduces to the celebrated \textsc{LinUCB} (or \textsc{Oful}) algorithm of \cite{Abbasi11}.
In addition, however, it \emph{can handle contextual misspecified linear bandits while retaining computationally tractability}, making this the first algorithm and analysis for this setting, although we require knowledge of the misspecification level. A similar result was recently derived for a different algorithm based on $G$-experimental design  \cite{lattimore2020learning} for the more restrictive setting of non-contextual (i.e., with features not depending on the state and fixed action space) misspecified linear bandits; however, their approach is agnostic to the misspecification level.

\textbf{Core ideas.}  
LSVI-based algorithms have been successfully analyzed for low-rank MDPs \cite{jin2020provably} by adding exploration bonuses at every experienced state, thereby ensuring optimism by backward induction.
In contrast, our more general setting demands that the value function stays linear, ruling out approaches based on exploration bonuses. In fact, if the value function used for backup is not linear, low inherent Bellman error does not provide any guarantee about how errors may propagate, which can be exponential in the general case \cite{zanette19limiting}.

Our proposal extends the LSVI algorithm to return an optimistic solution at the initial state through \emph{global} optimization over the value function parameters, while still enforcing linearity of the representation. This has two advantages: 1) (\emph{handling of the bias}) it enables us to use the concept of inherent Bellman error, requiring that the Bellman operator be applied to \emph{linear} action-value functions and avoiding a $\sqrt{d}$ amplification of the value function error at every step \cite{zanette19limiting}; 2) (\emph{handling of the variance}) it keeps the complexity of the action-value functional space small (linear), enabling the use of confidence intervals that are as tight as those used in the bandit literature, yielding the optimal finite-sample statistical rate.

\section{Notation}
We consider an undiscounted finite-horizon MDP~\cite{puterman1994markov} $M = (\mathcal{S}, \mathcal{A},p, r, H)$ with state space $\mathcal{S}$, action space $\mathcal{A}$, and horizon length $H \in\mathbb{N}^+$.
For every $t \in [H] \defeq \{1, \ldots, H\}$, every state-action pair is characterized by an expected reward $r_t(s,a)$ with an associated reward random variable $R_t(s,a)$ and a transition kernel $p_t(\cdot \mid s,a)$ over next state.
We assume $\mathcal{S}$ to be a measurable, possibly infinite, space and $\mathcal{A}$ can be any (compact) time and state dependent set (we omit this dependency for brevity).
For any $t \in [H]$ and $(s,a) \in \mathcal{S} \times \mathcal{A}$, the state-action value function of a non-stationary policy $\pi = (\pi_1, \ldots, \pi_H)$ is defined as $ Q^{\pi}_{t}(s,a) = r_{t}(s,a) + \mathbb{E} \left[ \sum_{l = t+1}^{H} r_{l}(s_{l}, \pi_{l}(s_l))\mid s,a \right]$
and the value function is $V^{\pi}_t(s) = Q^{\pi}_t(s, \pi_t(s))$.
Since the horizon is finite, under some regularity conditions, e.g.,~\cite{shreve1978alternative}, there always exists an optimal policy $\pi^\star$ whose value and action-value functions are defined as $\Vstar_t(s) \defeq V^{\pi^\star}_t(s) = \sup_{\pi} V^{\pi}_t(s)$ and $\Qstar_t(s,a) \defeq Q^{\pi^\star}_t(s,a) = \sup_{\pi} Q^{\pi}_t(s,a)$.

The value iteration (or backward induction) algorithm \cite{sutton2018reinforcement} computes $\pistar$ and $\Vstar$ as follows: it starts from $V^\star_{H+1}(s)=0$ for all $s\in\mathcal{S}$ and it computes $Q^\star_t$ using the Bellman equation in each state-action pair recursively from $t=H$ down to $1$ and it returns the optimal policy $\pi^\star_t(s) = \arg\max_a Q^\star_t(s,a)$. In particular, the Bellman operator $\T_t$ applied to  $Q_{t+1}$ is defined as
\begin{align*}
\T_t(Q_{t+1})(s,a) = r_t(s,a) + \E_{s' \sim p_t(s,a)} \max_{a'}Q_{t+1}(s',a').
\end{align*}

\section{Linear Value Function Frameworks}

In this section we introduce basic notation and assumptions for linear function approximation, we define the concept of inherent Bellman error, and we investigate connections with alternative settings.

Whenever the state space $\mathcal{S}$ is too large or continuous, value functions cannot be represented by enumerating their values at each state or state-action pair. A common approach is to define a feature map $\phi_t: \StateSpace  \times\ActionSpace \rightarrow \R^{d_t}$, possibly different at any $ t \in [H]$, embedding each state-action pair $(s,a)$ into a $d_t$-dimensional vector $\phi_t(s,a)$. The action-value functions are then represented as a linear combination between the features $\phi_t$ and a vector parameter $\theta_t\in\R^{d_t}$, such that $Q_t(s,a) = \phi_t(s,a)^\top \theta_t$. This effectively reduces the complexity of the problem from $|\StateSpace \times \ActionSpace|$ down to $d_t$. 

We define the space of parameters $\theta$ inducing uniformly bounded action-value functions
\begin{align}
\label{eqn:Bt}
\mathcal B_t \defeq \{\theta_t \in \R^{d_t} \mid |\phi_t(s,a)^\top \theta_t | \leq D, \forall (s,a) \}.
\end{align}
We will later require the constant $D\in\R$ to be chosen to satisfy Asm.~\ref{ass:MainAssumption}. For instance, $D = 1$ requires the value function to be in $[-1,+1]$ and complies with the assumption.

Each parameter $\theta$ identifies an (action) value function 
$$Q_t(\theta_t)(s,a) = \phi_t(s,a)^\top \theta_t, \quad V_t(\theta_t) = \max_{a}\phi_t(s,a)^\top \theta_t$$
and the associated functional spaces
\begin{align}
\label{eqn:Qt}
\mathcal Q_t & \defeq \{Q_t(\theta_t) \mid \theta_t \in \mathcal B_t \}, \; \mathcal V_t \defeq \{ V_t(\theta_{t}) \mid    \theta_t \in \mathcal B_t \}.
\end{align}

\textbf{Inherent Bellman error.}
The value iteration algorithm can be used to compute an optimal policy \cite{sutton2018reinforcement} and it smoothly extends to linear approximators. The procedure repeatedly applies the Bellman operator $\T_t$ to an action-value function\footnote{One can reason with either the value function $V$ or the action-value function $Q$.} $Q_t\in \mathcal Q_t$ and projects the computed point $\T_t Q_t$ back to $\mathcal Q_{t+1}$ using a (e.g., least-squares) projection operator $\Pi_t$. The projection error is precisely the inherent Bellman error, which can be thought of as how \textit{close} the space $\mathcal Q_t$ is w.r.t.\ the Bellman operator $\T_t$.

\begin{definition}
\label{def:InherentBellmanError}
The inherent Bellman error\footnote{A different definition, more suitable for  generative models with stationary policies using a $p$-norm induced by the sampling distribution is provided by \cite{munos2008finite}.} of an MDP with a linear feature representation $\phi$ is denoted with $\IBE$ and is the maximum over the timesteps $t\in[H]$ of 
	\begin{align*}
	\sup_{\theta_{t+1} \in \mathcal B_{t+1}}\inf_{\theta_{t} \in \mathcal B_{t}} \sup_{(s,a)\in \StateSpace\times\ActionSpace}|\phi_t(s,a)^\top \theta_t \\
	- \(\T_t Q_{t+1}(\theta_{t+1})\)(s,a)|.
	\end{align*}
\end{definition}
Our definition of inherent Bellman error is \emph{natural} in the sense that it is defined with respect to the linear action-value function class without additional clipping if the value function exceeds a prescribed threshold and is not enlarged to incorporate exploration bonuses (see e.g., \cite{wang2019optimism}). Alternative definitions may enlarge the underlying functional space in an artificial, non linear, possibly algorithm-dependent way, and result in a much more restrictive definition of inherent Bellman error. We notice that while our definition is less restrictive, it rules out traditional forms of exploration based on \emph{adding exploration bonuses}, making it harder to design effective exploration strategies.

\textbf{Properties.} We discuss the properties of MDPs with $\IBE = 0$. An immediate consequence of~def.~\ref{def:InherentBellmanError} is that when $\IBE = 0$ the reward function is linear, and so is the transition kernel \emph{when applied to elements of} $\mathcal V_{t+1}$.

\begin{restatable*}[Linearity of Rewards and Restricted Linearity of Transitions]{prop}{linearity}
\label{prop:linearity}
Given an MDP and a linear feature representation with $\mathcal B_t = \R^{d_t}$ and inherent Bellman error $\IBE = 0$ we have that the rewards are linear in the sense that:
\begin{align*} 
\inf_{\theta^R_{t} \in \mathcal B_{t}} \sup_{(s,a)\in \StateSpace\times\ActionSpace}|r_t(s,a) - \phi_t(s,a)^\top \theta^R_t| = 0
\end{align*}
and the transition have a linear effect on members of $\mathcal V_{t+1}$
	\begin{align*}
		 \sup_{\theta_{t+1} \in \mathcal B_{t+1}} \inf_{\theta^P_{t} \in \mathcal B_{t}} \sup_{(s,a)\in \StateSpace\times\ActionSpace}| \E_{s'\sim p_t(s,a)} V_{t+1}(\theta_{t+1})(s') \\
		 - \phi_t(s,a)^\top  \theta_t^P| = 0.
\end{align*}
\end{restatable*}

If $\IBE = 0$, the application of the Bellman operator $\T_t$ to members of $\mathcal Q_{t+1}$ always produces a member of $\mathcal Q_{t}$, i.e., $\T_t \mathcal Q_{t+1} \subseteq \mathcal Q_t$. 
From here, we can immediately see that the zero inherent Bellman error assumption is more general than low-rank MDPs  \cite{yang2020reinforcement,jin2020provably,zanette2020frequentist}. Indeed, in low-rank MDPs the Bellman operator returns a function in the range of the features (i.e., in $\mathcal Q_t$) \emph{regardless of value function $Q_{t+1}$}, while problems with zero inherent Bellman error are only required to map elements of $\mathcal Q_{t+1}$ to $\mathcal Q_{t}$, and are thus more general approximators.

\begin{restatable*}[Low Rank $\subseteq$ LSVI Conditions]{prop}{LowrankVsNOIBE}
\label{prop:LowrankVsNOIBE}
Let $\mathcal B_t = \R^{d_t}$, and consider an MDP with associated linear feature representation $\phi$. If the MDP is a low rank (or linear) MDP, i.e., for a parameter $\theta^R_t\in \R^{d_t}$ and a measure function\footnote{a positive function such that $\| \Psi_t\|_{TV} = 1$} $\psi_t(\cdot)$:
\begin{align*}
	\forall (s,a,t,s'), \quad r_t(s,a) & = \phi_t(s,a)^\top\theta^R_t   \\
	p_t(s'\mid s,a) & = \phi_t(s,a)^\top\psi_t(s') \\
\numberthis{\label{eqn:LinearMDPequations}}
\end{align*}
then $\IBE = 0$. However, the converse does not hold, i.e., there exists an MDP and a linear feature extractor $\phi$ with $\IBE = 0$ which is not a linear MDP in the sense of \cref{eqn:LinearMDPequations}.
\end{restatable*}

Another assumption often made on the approximation space is that the action-value functions for \emph{all policies} do belong to $\mathcal Q_t$ (LSPI condition), a condition normally employed to show convergence of LSPI~\cite{lagoudakis2003least}. This assumption is also strictly less restrictive than low-rank (see also \cite{jin2020provably} for a claim in one direction).

\begin{restatable*}[Low Rank $\subseteq$ LSPI Conditions]{prop}{LowrankVsLP}
\label{prop:LowrankVsLP}
If a given MDP is low rank in the sense of  \cref{eqn:LinearMDPequations} then the value function of all policies admit a linear parameterization: 
\begin{align*}
	\forall \pi, \; \forall t \in [H], \; \exists \theta_t^\pi \quad \text{such that} \quad Q^{\pi}_t(s,a) = \phi_t(s,a)^\top\theta_t^\pi.
\end{align*}
However, there exists an MPD and a linear approximator with feature extractor $\phi$  which satisfies the above display but there exists no $\psi_t$ such that \cref{eqn:LinearMDPequations} holds.
\end{restatable*}

One may wonder what is the relation between MDPs with no inherent Bellman error and MDPs where all action-value function for all policies are linear, i.e., the LSVI and LSPI conditions. These are two very distinct assumptions: the former deals with policies \emph{that are optimal with respect to a parameter}, while the latter deals with arbitrary policies. Conversely, the latter deals with the $Q$ values that actually corresponds to $Q$ values of policies, while the former measures the error with respect to any function in the class.

\begin{restatable*}[LSVI Conditions $\neq$ LSPI Conditions]{prop}{NOIBEvsLP}
\label{prop:NOIBEvsLP}
	There exists an MDP and a linear representation with feature extractor $\phi$ with $\IBE = 0$
and yet the policies are not linearly parameterizable in the sense that:
\begin{align*}
	\exists \pi, \exists t \in [H], \not \exists \theta_t^\pi\in\R^{d_t} \quad \text{s.t.} \quad Q_t^{\pi} =  \phi_t(s,a)^\top \theta_t.
\end{align*}

Vice-versa, there exists an MDP and a feature representation such that all action-value functions of all policies admit a linear parameterization: 
\begin{align*}
	\forall \pi, \forall t\in[H], \; \exists \theta_t^\pi \quad \text{that satisfies} \quad Q^{\pi}_t(s,a) = \phi_t(s,a)^\top\theta_t^\pi
\end{align*}
and yet the inherent Bellman is non-zero: $
	\IBE > 0$.
\end{restatable*}

The final connection we make is with settings with \emph{low Bellman rank}, see \cite{jiang17contextual}. It is possible to show that if the LSVI conditions are satisfied, the Bellman rank is at most $d$, where $d$ is the dimensionality of the features. However, no statistically efficient algorithm exists for this setting, because \textsc{Olive} from \cite{jiang17contextual} has an explicit dependence on the size of the action space, which can be very large or infinite in the setting we consider here.

\section{Algorithm}
We consider the standard online learning protocol in finite-horizon problems, where at each episode $k$, the learner executes a policy $\pi_k$, records the samples in the trajectory, updates the policy and reiterates over the next episode. We first recall the standard LSVI. At the beginning of episode $k$, consider timestep $t$ and assume the next-step parameter is fixed and equal to $\theta_{t+1}$. The objective function of the regularized least-square is
\begin{align}
\label{eqn:RegObj}
\sum_{i=1}^{k-1} \(\phi_{ti}^\top \theta - r_{ti} - V_{t+1}(\theta_{t+1})(s_{t+1,i})\)^2 + \lambda \| \theta \|^2_2
\end{align} where $\{\phi_{ti}\}_{i=1,\dots,k-1}$ are the features observed at timestep $t$ in state $s_{ti}$ and $r_{ti}$ are the corresponding rewards. For any $\lambda > 0$ the prior display has a closed-form solution
\begin{align}
	  \widehat \theta_{t} = \Sigma^{-1}_{tk}\sum_{i=1}^{k-1} \phi_{ti} \Big[ r_{ti} + V_{t+1}(\theta_{t+1})(s_{t+1,i})\Big]
\end{align}
with 
$
	\Sigma_{tk} \defeq \sum_{i=1}^{k-1}\phi_{ti}\phi_{ti}^\top + \lambda I
$
as the empirical covariance.

We introduce an optimistic variant of LSVI, where the optimistic parameters are chosen by solving a global optimization problem across the whole horizon $H$. At each episode, \Alg{} (in Alg.~\ref{alg:AlgoLabel}) solves the following problem.

\begin{definition}[Planning Optimization Program]
	\label{def:PlanningOptimizationProgram}
	\begin{align*}
	& \max_{\substack{(\overline \xi_1,\dots,\overline \xi_H) \\ (\widehat \theta_1,\dots,\widehat \theta_H) \\
			(\overline \theta_1,\dots,\overline \theta_H)
	}} \quad \max_a \phi_1(s_{1k},a)^\top \overline \theta_{1}  \quad \textrm{subject to}  \\
	&  \widehat \theta_{t} =  \Sigma_{tk}^{-1}\sum_{i=1}^{k-1}\phi_{ti}^\top \(r_{ti} + V_{t+1}(\overline\theta_{t+1}) (s_{t+1,i})\) \\
	& \overline \theta_{t} = \widehat \theta_{t} + \overline \xi_{t}; \quad  \|\overline \xi_{t}\|_{\Sigma_{tk}} \leq \sqrt{\alpha_{tk}}; \quad  \overline \theta_{t} \in \mathcal B_t 
	\end{align*}
\end{definition}

As we will show in the technical analysis, a feasible solution $(\thetastar_1,\dots,\thetastar_H)$, corresponding to the best approximator (in \cref{eqn:main.thetastar}) always exists and so the program is well posed.

The least-square solution $\widehat \theta_t$ is used as a constraint and perturbed by 
adding a vector $\overline \xi_{t}$ as optimization variable,\footnote{We add the subscript $k$ later to indicate the actual variable chosen by the optimization procedure in episode $k$.} subject to
\begin{align}
\label{eqn:main.alpha}
	\| \overline \xi_{t} \|_{\Sigma_{tk}} \leq \sqrt{\alpha_{tk}} := \underbrace{\sqrt{\beta_{tk}}}_{\text{noise}} + \underbrace{\vphantom{\sqrt{\beta_{tk}}} \sqrt{\lambda}\Radius_t}_{\text{regularization}} + \underbrace{\vphantom{\sqrt{\beta_{tk}}}\sqrt{k}\IBE,}_{\text{misspec.}}
\end{align}
where $\alpha_{tk}$ is designed to account for the noise, misspecification, and regularization bias. The actual bound is a function of the allowable radius $\Radius \leq \sqrt{d_t}$ for the parameter (as in assumption \ref{ass:MainAssumption}) and the noise parameter $\sqrt{\beta_{tk}} = \widetilde O(\sqrt{d_t})$ stems from self-normalizing concentration inequalities as described in the technical analysis later, while $\IBE$ is the inherent Bellman error. The resulting parameter $\overline \theta_{t} = \widehat \theta_{t} + \overline \xi_{t}$ must satisfy the constraint $ \overline \theta_{t} \in \mathcal B_{t}$. This is equivalent to clipping the value function to avoid out-of-range values, with the difference that such clipping occurs directly in the parameter space as opposed to state by state, and thus preserves linearity. 
 
We emphasize that the optimization over the $\overline \xi_t$'s is \emph{global},
 in stark contrast to the tabular setting and even the setting of linear MDPs considered by \cite{yang2020reinforcement,jin2020provably}, where any perturbation (clipping, exploration bonus, etc) can be done state by state. For example, \cite{jin2020provably} define $
  	\Qbar_t(s,a) \stackrel{\text{redefined}}{=}  \min\{1, \phi_t(s,a)^\top \overline \theta_t + \textsc{Bonus}\}$
 where the bonus is the result of maximizing $\overline \xi_t$ state by state. This trick works in the low-rank setting of \cite{jin2020provably}, since any non-linear component is filtered out by the low-rank projector. \Alg{} instead pushes that maximization over the $\overline \xi_t$'s ``outside'' of local states, i.e., it performs a \emph{global maximization} to ensure linearity of the value function representation, a mandatory condition in our setting to avoid an exponential propagation of the errors. 
 
 When linear representations are enforced, however, the algorithm cannot choose a value function everywhere optimistic due to values in different states possibly being negatively correlated. \Alg{} shoots for being optimistic at the initial state, but in general the algorithm does not play optimistic actions in the encountered states at later timesteps. Fortunately, this is enough to attain a rate-optimal efficiency.
 
\begin{algorithm}[htb]
   \caption{\Alg{}}
   \label{alg:AlgoLabel}
\begin{algorithmic}[1]
\STATE Input: failure probability $\delta$,  regularization $\lambda = 1$, feature extractor $\phi$, inherent Bellman residual $\IBE$
\STATE Initialize $\Sigma_{t1} = \lambda I$, for $t = 1,2,\dots,H$. 
\FOR{$k = 1,2,\dots$}
\STATE Receive starting state $s_{1k}$
\STATE Set $ \overline{\theta}_{H+1,k} = \widehat{\theta}_{H+1,k} = \overline \xi_{H+1,k}  = 0$
\STATE Solve program of  \fullref{def:PlanningOptimizationProgram}.
\STATE Execute $ \pi_{k}:  (s,t) \mapsto \argmax_{a} \phi_{t}(s,a)^\top\overline \theta_{tk}$ and collect $ (s_{tk}, a_{tk}, r_{tk})$ for $ t\in [H]$.
\ENDFOR
\end{algorithmic}
\end{algorithm}
Although \Alg{} is proved to be near optimal, it is difficult to implement the algorithm efficiently. This should not be seen as a fundamental barrier, however. The issue of computational tractability arises even for tabular problems \cite{Bartlett09,zhang2019regret}, but of course the problem is more pronounced when function approximators are implemented \cite{krishnamurthy2016pac,jiang17contextual,sun2018model,osband2014near}, and even for low-rank MDPs the first regret result has been obtained at the expense of a practical algorithm \cite{yang2020reinforcement}. Fortunately, later work has made progress on the computational aspects for many of these settings \cite{tossou2019near,Fruit18,dann2018oracle,jin2020provably}.  
For now, we leave this to future work.

\textbf{Relaxations.} 
With an eye towards a possible relaxation, we notice that the constraint $\overline \theta_t \in \mathcal B_t$ can be expensive to evaluate because it would require checking that every product $\phi_t(\cdot,\cdot)^\top \overline \theta_t$ is bounded. However, one can use simpler, more restrictive geometries and assume $\mathcal B_t$ is a unit ball, bypassing this problem. The algorithm regret bound for this case is the same as that of  \cref{thm:MainResult}.

Finally, it is possible to avoid the regularization in the least square objective of \cref{eqn:RegObj} and relax the requirement $\|\overline \theta_t \|_2\leq \sqrt{d_t}$ as presented later in assumption \ref{ass:MainAssumption}. In fact, the constraint on $\mathcal B_t$ suffices to avoid ill-conditioned solutions, but then one would need to resort to pseudo-inverse computations \cite{auer2002using}, making the algorithm / analysis more complicated.

\section{Main Result: Regret Upper Bound}

\begin{assumption}[Main Assumption]
\label{ass:MainAssumption}
We assume:
\begin{itemize}
	\item $ |Q_t^{\pi}(s,a)| \leq 1, \quad \forall \pi, \forall (s,a,t)$
	\item $\| \phi_t(s,a)\|_2\leq \Lphi \leq 1, \quad \forall (s,a,t)$
		\item For any $Q_t\in\mathcal Q_t$ and any $(s,a,t) \in \StateSpace\times\ActionSpace\times [H]$ define the random variable\footnote{Here, $R_t(s,a)$ is the reward random variable, and $s'\sim p_t(s,a)$ is the successor state random variable under the distribution $p_t(s,a)$.} 
		$X = R_t(s,a) + \max_{a'}Q_{t+1}(s',a') $. Then the noise $\eta = X - \E X$ is $1$-subgaussian
		\item $\forall t\in [H], \forall \theta_t \in \mathcal B_t,$ it holds that $ \| \theta_t \|\leq \Radius_t \leq \sqrt{d_t}$, and $\mathcal B_t$ is compact
\end{itemize}
\end{assumption}
The first condition is a condition on the scaling of the problem and the bound on the feature norm is without loss of generality. The sub-Gaussianity is standard already for linear bandits \cite{Abbasi11,lattimore2020bandit}. In particular, if the reward are in $[0,1]$ and $D = 1$ in \cref{eqn:Bt}, which gives $\Vbar(\cdot) \in [-1,1]$, then this condition is automatically satisfied. Finally, the bound on the parameter limits the bias introduced by regularization which scales with the norm of the parameter, but a psedoinverse computation would relax this requirement. 

After rescaling, however, our assumptions are much weaker the the usual setting that requires $r_t(\cdot,\cdot) \in [0,1]$ and $V_t^\pi(\cdot) \in [0,H]$ since we allow the reward to be of the same order as the value function after rescaling and even be negative. This is a harder setting \cite{jiang2018open,zanette2019tighter}.

\begin{restatable*}[Main Result]{thm}{MainResult}
\label{thm:MainResult}
	Under assumption \ref{ass:MainAssumption} with $\lambda = 1$, 
	with probability at least $1-\delta$ jointly over all episodes it holds that the regret of \Alg{} is bounded by:
	\begin{align*}
		\textsc{Regret}(T) = \widetilde O(\underbrace{\sum_{t=1}^H d_t  \sqrt{K}}_{\textrm{variance term}} + \underbrace{ \sum_{t=1}^H \sqrt{d_t} \IBE K}_{\textrm{approximation term}}).
	\end{align*} 
\end{restatable*}

There are no additional ``lower order'' terms in the above display, although the $\widetilde O(\cdot)$ notation hides, as usual, logarithms of $d_t,H,K,1/\delta$.

Care must be taken when comparing across settings with different scaling. In particular, \emph{rescaling the problem} (i.e., the reward function) \emph{by $H$} increases the sub-Gaussian norm of the rewards and transitions, and the value of the inherent Bellman error alike, yielding \emph{an extra $H$ factor in the regret bound}. For example, in the setting that the rewards are bounded in $[0,1]$ and the value function is in $[0,H]$ with $d_1 = \dots = d_H \defeq d$ and $\IBE = 0$ for simplicity, the above regret bound reduces (with $T = KH$) to $\widetilde O(dH^{\frac{3}{2}}\sqrt{T})$. 

\textbf{Low-rank MDPs}
As explained in \cref{prop:LowrankVsNOIBE}, our result applies to low-rank MDPs; surprisingly, this shows that at least $\sqrt{d}$ improvement is possible in the main rate compared to the best-known $\widetilde O((dH)^{3/2}\sqrt{T})$ of \cite{jin2020provably} upper bound despite \Alg{} is not specifically tailored to handle low-rank MDPs. 
This is possible because \Alg{} looks for optimistic solutions directly in the $\theta$ parameter space instead of perturbing the value function by an exploration bonus as in \cite{jin2020provably}. When the value function is perturbed by a bonus, it grows in complexity as it departs from the linear space; this requires an additional union bound over a more complicated value function class and ultimately loses a $\sqrt{d}$ factor. Finally, the inherent Bellman error covers the notion of approximate low-rank MDPs \cite{jin2020provably}, and on the misspecification regret term we save a $\sqrt{d}$ factor as well thanks to a more careful projection argument in \cref{lem:ProjectionBound}.

\section{Contextual Misspecified Linear Bandits}
Our framework reduces to bandits with linear approximators when $H=1$ (we drop the time subscript $t$ in this case):  \Alg{} can handle \emph{contextual misspecified linear bandits}, where contextual refers to allowing the action set to change as the feature extractor can be a  function of the context. It follows from the definition that the inherent Bellman error is the reward function misspecification in this case.
\begin{corollary}[\textsc{LinUCB} Regret on Contextual Misspecified Linear Bandits]
\label{cor:ConMissLinBand}
Consider a misspecified contextual linear bandit problem with reward response
\begin{align*}
	r(s,a) = \phi(s,a)^\top \thetastar + \eta + f(s,a)
\end{align*}
with $|\phi(s,a)^\top \thetastar| \leq 1$, $\|\thetastar\|_2 \leq \sqrt{d}$, $\| \phi(s,a) \|_2 \leq 1$,  misspecification $|f(s,a)|\leq \IBE$ and $1$ sub-Gaussian noise $\eta$. If \Alg{} is informed that $H=1$ then the algorithm reduces to the \textsc{LinUCB} (aka \textsc{Oful}) algorithm of \cite{Abbasi11} with arm selection strategy $
		\argmax_{a \in \ActionSpace, \|\overline \xi\|_{\Sigma_{k}} \leq \sqrt{\alpha_k}} \phi(s_{k},a)^\top \( \widehat \theta_{k} + \overline \xi_k \)
	$
	but a different confidence interval: $
		\| \overline \theta_k - \widehat \theta_k \|_{\Sigma_k} = \| \overline \xi_k \|_{\Sigma_k} \leq \sqrt{\alpha_k}.
	$
	The arm selection strategy admits the closed-form solution $
		\argmax_{a \in \ActionSpace} \Big[ \phi(s_{k},a)^\top \widehat \theta_{k} + \| \phi(s_{k},a)\|_{\Sigma^{-1}_{k}}\sqrt{\alpha_{k}} \Big]
	$
	and the algorithm has a high probability regret bound
	\begin{align*}
		 \widetilde O\(d \sqrt{K} + \sqrt{d} \IBE K \).
	\end{align*}
\end{corollary}
\vspace{-0.4cm}
The corollary above is immediate upon substituting $H=1$ in \cref{thm:MainResult} and verifying that our assumptions match the setting described in the corollary, which is the standard linear bandit setting\footnote{We drop the constraint $\theta \in \mathcal B$ for simplicity}
 \cite{lattimore2020bandit} with the addition of misspecification (few more details in \cref{sec:Bandits}).
 
 Due to the equivalence to \textsc{LinUCB} the algorithm is computationally tractable when applied to bandits; the \emph{key} difference with vanilla \textsc{LinUCB} resides in the width of the confidence intervals, parameter $\alpha_k$. In the absence of misspecification ($\IBE = 0$), $\sqrt{\alpha_k} = \sqrt{\beta_k} + \sqrt{\lambda}\Radius = \widetilde O(\sqrt{d})$, as in the work of \cite{Abbasi11}. When misspecification is present, however, there is a correction factor $\sqrt{k}\IBE$ in the definition of $\sqrt{\alpha_k}$, see equation \cref{eqn:main.alpha}. In other words, this is the factor one should add to the exploration bonus for an \text{LinUCB}-like algorithm in case of (potentially adversarial) misspecification.

The recent result by \cite{du2019good} applies here (see also the work of \cite{van2019comments}). They show that for large misspecification $\IBE \gtrapprox 1/\sqrt{d}$ an exponential sample complexity is unavoidable to identify an arm with positive return.  This does not contradict our result, because our regret is $\widetilde O(K)$ under such large misspecification, which is vacuous as the maximum loss up to episode $K$ is exactly $K$.

Notice that the equivalence is established by informing \Alg{} of the setting (through the horizon $H=1$) unlike \cite{Zanette18a}. Finally, if the corruption $f(\cdot)$ is only a function of the context then it is possible to do much better \cite{krishnamurthy2018semiparametric}.

This surprising connection with the popular \textsc{LinUCB} makes \Alg{} (or \textsc{LinUCB} with a correction on the exploration bonus) the first algorithm capable of handling misspecified \emph{contextual} linear bandits, although we are not the first to consider misspecification in linear bandits per se: \cite{ghosh2017misspecified} propose an algorithm that switches to tabular if misspecification is detected and \cite{gopalan2016low} consider the case that the misspecification is less than roughly  the action gap; \cite{van2019comments} comment on the lower bound by \cite{du2019good} using the Eluder dimension. Finally, \cite{lattimore2020learning}
have recently obtained a result similar to ours, but for a different setting. Their algorithm can leverage having finitely many actions (where a $\sqrt{d}$ factor can be saved; otherwise their regret is the same as ours) but relies heavily on $G$-experimental design: the algorithm will not work without a stationary action set, ruling out the important case of contextual linear bandits where the action is allowed to depend on the context. However, our correction to vanilla \textsc{LinUCB} relies on having knowledge of the misspecification, while the approach of \cite{lattimore2020learning} is agnostic. Furthermore, concurrently to our work \cite{lattimore2020learning} also consider the same modification to \textsc{LinUCB} as we do here, and provide proof that the algorithm can fail if no modification is implemented. However, these definitions of misspecification are adversarial in nature, and for less pathological problems the algorithm is expected to perform well.

\section{Lower Bounds}

In terms of statistical rate, \Alg{} is unimprovable due to a lower bound directly borrowed from the bandit literature.

\begin{restatable}[Lower Bound Without Misspecification]{prop}{LowerBoundNoMisspec}
\label{prop:LowerBoundNoMisspec}
Let $\widetilde d \defeq \sum_{t=1}^H d_t$.
There exist a class of $H$-horizon MDPs that satisfy asm. \ref{ass:MainAssumption} and $H$ feature maps $\phi_t(\cdot,\cdot)\in \R^{d_t}$, with $\widetilde d \geq 2H$ such that for $K = \Omega( \widetilde d^2)$ the expected regret of any algorithm is $\Omega(\widetilde d\sqrt{K})$.
\end{restatable}

The fact that our result matches the lower bound can appear surprising, because our work relies on a sub-Gaussian conditions and disregards the variance in the process. It does not use a ``law of total variance'' argument \cite{Azar12,Azar17}, which was necessary in the past to obtain rate-optimal algorithms for tabular settings. One may wonder whether a $\sqrt{H}$ factor can be saved by that argument for MDPs parameterized by linear action-value function. Due to the bandit lower bound, no such improvement is possible with linear function approximations, unless the structure is restricted further. The reason is that our setting is a superset of tabular RL \cite{Azar17} and contains harder instances than the lower bound for tabular RL (in particular, a linear bandit problem at a single timestep) but the law of total variance would bring no benefit to those structures.

\textbf{Approximation error}
Our positive result regarding misspecification matches the LSPI analysis of \cite{lattimore2020learning} but for the harder \emph{online} setting. Although the two respective frameworks (i.e., LSPI vs LSVI conditions) are incompatible as explained in \cref{prop:NOIBEvsLP}, we notice a similar effect: a square-root factor of the problem dimensionality multiplies the ``misspecification'' error. While the LSPI analysis of \cite{lattimore2020learning} relies on having features from $G$-optimal design to query the system, \emph{in the online setting we're not free to choose arbitrary features anywhere in the state-action space}. As a result, the agent can learn on an ill-conditioned basis, and the prediction error on features much different from those experienced can be very large. Our analysis shows that while this can indeed be the case, the situation of high prediction error cannot persist for too long and the $\sqrt{d}$ loss in prediction accuracy is, \emph{on average}, recovered.
Using the recent result by \cite{du2019good}, we can augment \cref{prop:LowerBoundNoMisspec} by including a sequence of misspecified linear bandits, obtaining the following result (see also \cref{sec:LowerBounds}): 

\begin{restatable*}[Lower Bound for Inherent Bellman Error Setting]{thm}{LowerBound}
\label{thm:MasterLowerBound}
There exist feature maps $\phi_1,\dots,\phi_H$ that define an MDP class $\mathcal M$ such that every MDP in that class satisfies assumption \ref{ass:MainAssumption} with inherent Bellman error $\IBE$ and such that the expected regret of any algorithm on at least a member of the class (for $A \geq 3, d_t \geq 3, K = \Omega((\sum_{t=1}^H d_t)^2)$) is $\Omega(\sum_{t=1}^H d_t \sqrt{K} + \sum_{t=1}^H \sqrt{d_t} \IBE K)$, that is:
\begin{align*}
& \min_{\mathscr A}\max_{M \in \mathcal M} \sum_{k=1}^K \( \Vstar_1 - V_1^{\pi_k}\) (s_{1k}) \\
& =	\Omega(\sum_{t=1}^H d_t \sqrt{K} + \sum_{t=1}^H \sqrt{d_t} \IBE K).
\end{align*}	
\end{restatable*}

\section{Proof Overview}
We now give a quick proof sketch and highlighting how working in the parameter space allows us to 1) avoid an exponential propagation of the errors by leveraging the notion of inherent Bellman error (handling of the bias) and 2) preserve confidence intervals that are as tight as in a bandit problem (handling of the variance). Our objective is to bound the regret: $\textsc{Regret}(K) \defeq \sum_{k=1}^K \(\Vstar_1 - V_1^{\pi_k}\)(s_{1k})$ for the chosen policies $\pi_k$, but first we need to discuss how the errors propagate and how to ensure optimism.

\subsection{Propagation of errors}
The inherent Bellman error condition  ensures that there exists a parameter $\mathring \theta_{t}$ and a Bellman residual function $\mathring \Delta_{t}$, both depending on $\Qbar_{t+1}$, such that $\mathring \Delta_{t}(\Qbar_{t+1})(s,a)=$ 
\begin{align}
\label{eqn:BEcond}
  = \phi_t(s,a)^\top \mathring \theta_{t}(\Qbar_{t+1}) - \( \T_t \Qbar_{t+1} \)(s,a)
 \end{align}
with $\|\mathring \Delta_{t}(\Qbar_{t+1}))\|_\infty \leq \IBE$ \emph{provided that} $\Qbar_{t+1}\in \mathcal Q_{t+1}$. 
In other words, we can successfully represent $\T_t \Qbar_{t+1}$ up to an additive error $\IBE$ \emph{if the next-step $\Qbar_{t+1}$ function is linear}. 

This representational constraint unfortunately rules out adding exploration bonuses as in
prior low-rank work \cite{yang2020reinforcement,jin2020provably} as well as in tabular MDPs; their addition can have the backup $\T_t\Qbar_{t+1}$ leave the linear space (which is equivalent to having large $\IBE$) and can lead to divergence of the repeated least-square  procedure \cite{baird1995residual,sutton2018reinforcement,zanette19limiting}.

\textbf{Error decomposition}
We aim to compute the error encountered in minimizing \cref{eqn:RegObj} with $V_{t+1} = \Vbar_{t+1}$ fixed and no regularization. Denote with $s_{ti}$ the $i$-th state encountered at timestep $t$ of episode $i$, and let $a_{ti} = \pi_{ti}(s_{ti})$. Define the $i$-th sample noise  
$\eta_{ti}(\Vbar_{t+1})\defeq r_{ti} - r_t(s_{ti},a_{ti}) + \Vbar_{t+1}(s_{t+1,i}) -  \E_{s'\sim p_t(s_{ti},a_{ti})} \Vbar_{t+1}(s')$ and the misspecification $\mathring \Delta_{ti}(\Qbar_{t+1}) \defeq \mathring \Delta_{t}(\Qbar_{t+1})(s_{ti},a_{ti}) $.
Premultiply $\widehat \theta_{tk} $ (which minimizes \cref{eqn:RegObj}) by $\phi_{t}(s,a)^\top$ and use the definitions just introduced: $ \phi_t(s,a)^\top\widehat \theta_{tk} =$
\begin{align*}
& \phi_t(s,a)^\top\Sigma_{tk}^{-1}\sum_{i=1}^{k-1}\phi_{ti} \(\T_t\Qbar_{t+1}(s_{ti},a_{ti}) + \eta_{ti}(\Vbar_{t+1}) \) \\
& = \phi_t(s,a)^\top \Big[ \mathring \theta_{t}(\Qbar_{t+1}) + \\
& + \Sigma_{tk}^{-1}\sum_{i=1}^{k-1}\phi_{ti} \Bigm(\mathring \Delta_{ti} + \eta_{ti}\Bigm) \(\Qbar_{t+1}\) \Big]\\
& \stackrel{\cref{eqn:BEcond}}{=}  \T_t(\Qbar_{t+1})(s,a) + \mathring \Delta_{t} (\Qbar_{t+1})(s,a) + \\ 
& +  \phi_t(s,a)^\top\Sigma_{tk}^{-1}\sum_{i=1}^{k-1}\phi_{ti} \Bigm(\mathring \Delta_{ti} + \eta_{ti}\Bigm) \(\Qbar_{t+1}\).
\numberthis{\label{eqn:FirstDecomposition}}
\end{align*}
We discuss the main error terms below.

\textbf{Inherent Bellman error}
Cauchy-Schwartz and a projection argument (\cref{lem:ProjectionBound}) gives:
\begin{align*}
	|\phi_t(s,a)^\top\Sigma_{tk}^{-1}\sum_{i=1}^{k-1}\phi_{ti} \mathring \Delta_{ti}(\Qbar_{t+1})| \leq \| \phi_t(s,a) \|_{\Sigma^{-1}_{tk}} \sqrt{k}\IBE.
\end{align*}
The inability to correctly represent the application of the Bellman operator could be exploited adversarially to introduce an error that grows with $\sqrt{k}$ (where $k$ is the number of episodes). 
On average, however, the $\Sigma^{-1}_{tk}$-norm of those features that are selected shrinks as $\| \phi_t(s,a)\|_{\Sigma_{tk}^{-1}} \approx \sqrt{d_t/k}$. While the agent can select a $(s,a)$ pair where the product $\| \phi_t(s,a)\|_{\Sigma_{tk}^{-1}} \sqrt{k}\IBE$ can be large, this cannot happen for too long. Intuitively, a large prediction error is made only on features that are significantly different from those seen in the past, but trying those features reveals the correct prediction, which decreases the prediction error for that direction in the future.

\textbf{Noise error and covering argument}
 Cauchy-Schwartz again gives 
 \begin{align*}
 & |\phi_t(s,a)^\top\Sigma_{tk}^{-1}\sum_{i=1}^{k-1}\phi_{ti} \eta_{ti}(\Vbar_{t+1})| \\
 & \leq \| \phi_t(s,a) \|_{\Sigma^{-1}_{tk}} \|\sum_{i=1}^{k-1}\phi_{ti} \eta_{ti}(\Vbar_{t+1}) \|_{\Sigma^{-1}_{tk}} \\
 & \stackrel{def}{\leq}  \| \phi_t(s,a) \|_{\Sigma^{-1}_{tk}} \sqrt{\beta_{tk}}
 \end{align*}
 where $\beta_{tk}$ follows from the self normalizing bound of \cite{Abbasi11} modified to cover the functional space $\mathcal V_t$. The covering argument is necessary since the noise depends on $\Vbar_{t+1}$ which is itself random. More precisely, we can write
$
\sqrt{\beta_{tk}} \lessapprox \sqrt{\ln \det(\Sigma_{tk})^{\frac{1}{2}} + \ln \mathcal N}
$,
where $\mathcal N$ is the covering number to $\epsilon$ accuracy of $\mathcal V_{t+1}$. The determinant-trace inequality (see lemma 10 of \cite{Abbasi11}) bounds the volume of the covariance matrix  $
\ln\det(\Sigma_{tk})^{\frac{1}{2}}={\widetilde O(d_t)}$; fortunately the metric entropy $\ln \mathcal N$ is of the same order. To see this, remember that to cover $\mathcal V_t$ it is sufficient to cover $\mathcal B_t$, which is a $d_t$ dimensional object ($\subset \R^{d_t}$), and hence $\ln \mathcal N = \widetilde O(d_t)$. Therefore, despite having an additional union bound compared to \cite{Abbasi11} because of the moving target $\Vbar_{t+1}$, our confidence intervals are of the same order of magnitude.

This is the place where a $\sqrt{d_t}$ can be saved compared to for example \cite{jin2020provably,wang2019optimism}, which need to do a union bound over a more complicated function class  because of the exploration bonuses.

\textbf{Final expression}
Adding $\phi_t(s,a)^\top \overline \xi_t$ to both sides of \cref{eqn:FirstDecomposition} and using the bounds just derived gives $|\( \overline Q_{t} - \T_t\Qbar_{t+1}\)(s,a) | = $
\begin{align*}
& \leq \underbrace{\IBE}_{\substack{\textrm{misspecification} }} +  \|\phi_t(s,a)\|_{\Sigma^{-1}_{tk}} \times \\
& \Big(\underbrace{\sqrt{k}\IBE}_{\substack{\textrm{misspecification} }} + \underbrace{\sqrt{\alpha_{tk}}}_{\substack{\textrm{exploration}}} + \underbrace{\sqrt{\beta_{tk}}}_{\substack{\textrm{noise}}} \Big).
\numberthis{\label{eqn:DecompositionInMain}}
\end{align*}
It remains to define $\alpha_{tk}$, which controls the size of optimization parameters, justifying \cref{eqn:main.alpha}.

\subsection{Feasibility, best approximator and optimism}

A key point of optimistic approaches for exploration is to overestimate the value of policies by assigning them a statistically plausible return, and play the policy with the highest such value. 

Since the optimal value function is an upper bound to the value of all policies, technically an optimistic learner is only required to identify a policy with value at least as high as $\Vstar_1$ while satisfying some confidence intervals. 
To show it possible to achieve this with our formulation, we will find a feasible solution to the program of \cref{def:PlanningOptimizationProgram} that is ``close'' to $\Vstar$. 
In general $\Vstar_t \not \in \mathcal V_{t}$, and so we need to define the ``best'' approximator in $\mathcal V_{t}$ for $\Vstar_{t}$. 
We denote its parameter with $\thetastar_{t} \in \mathcal B_t$, inductively defined (see def. \ref{def:BestApproximant} in appendix) as the parameter one obtains by applying the \emph{exact} Bellman operator and then by minimizing the $\infty$ norm of the Bellman residual: $\thetastar_{t} \defeq $ 
\begin{align}
\label{eqn:main.thetastar}
	\argmin_{\theta \in \mathcal B_t}  \sup_{(s,a)} \Big| \phi_t(s,a)^\top \theta - \( \T_t Q_{t+1}(\thetastar_{t+1})\)(s,a) \Big|
	\end{align}
If $\IBE = 0$ then  $\phi_t(s,a)^\top \thetastar_t = \Qstar_t(s,a)$ inductively follows. 

\textbf{Computation of $\alpha_{tk}$}
Under an inductive argument, assume the program of \cref{def:PlanningOptimizationProgram} admits a partial solution $\overline \xi_{t+1},\dots,\overline \xi_{H}$ that satisfies $\overline \theta_{t+1} = \thetastar_{t+1},\dots,\overline \theta_{H} = \thetastar_{H}$ (the parameters for timesteps less than $t+1$ have not been decided yet). 

Now setting:
\begin{align}
\label{eqn:xidef_main}
	\overline \xi_t =
  - \Sigma_{tk}^{-1}\sum_{i=1}^{k-1}\phi_{ti} \Bigm( \mathring \Delta_{ti} + \eta_{ti}\Bigm)(Q_{t+1}(\thetastar_{t+1})) 
\end{align}
and adding $\phi_t(s,a)^\top \overline \xi_t$ back to \cref{eqn:FirstDecomposition} 
evaluated with $\Qbar_{t+1} = Q_{t+1}(\thetastar_{t+1})$ can ``undo'' the effect of noise and approximation error at timestep $t$,
producing (recall $\overline \theta_t = \widehat \theta_t + \overline \xi_t$) 
\begin{align*}
& \phi_t(s,a)^\top \overline \theta_{t} \\
& = \T_t(Q_{t+1}(\thetastar_{t+1}))(s,a) + \mathring \Delta_{t}(Q_{t+1}(\thetastar_{t+1}))(s,a).
	\end{align*}
	Comparing with \cref{eqn:main.thetastar} we can claim $\overline \theta_t = \thetastar_t$, completing the induction. Thus, the best approximator defined through $\thetastar_t$ is a feasible solution to the program of \cref{def:PlanningOptimizationProgram}. The corresponding value function $V_t(\thetastar_t)$  can make an error of size $\IBE$ in representing the Bellman backup, and this accumulates linearly, and hence \Alg{} is ultimately nearly-optimistic: 
\begin{align}
\label{eqn:NearOptimismInMain}
\Vbar_1(s_{1k}) \geq  \Vstar_1 (s_{1k})  - H\IBE.
\end{align}

As we'll see in a second, this near-optimism is enough to obtain a solid regret bound. 
Finally,
\cref{eqn:xidef_main} gives:
\begin{align}
	\| \overline \xi_{t} \|_{\Sigma_{tk}} \leq \underbrace{\Big\| \sum_{i=1}^{k-1}\phi_{ti} \Delta_{ti} \Big\|_{\Sigma_{tk}^{-1}}}_{\leq \sqrt{k}\IBE}  + \underbrace{\Big\| \sum_{i=1}^{k-1}\phi_{ti} \eta_{ti} \Big\|_{\Sigma_{tk}^{-1}}}_{\leq \sqrt{\beta_{tk}}} 
\end{align}
which matches \cref{eqn:main.alpha} after adding the regularization term.

\subsection{Regret Bound}
Finally, we can present the regret bound, which now follows similarly to prior analyses for model free algorithms (e.g., \cite{jin2018q}). Consider  the usual decomposition from the starting state $s_{1k}$: 
\vspace{-0.3cm}
$$\textsc{Regret}(K) \defeq 
 \sum_{k=1}^K \(\Vstar_1 - \overline V_{1k} + \overline V_{1k} -  V_1^{\pi_k}\)(s_{1k}).$$
The first term inside the parenthesis can be bounded by  \cref{eqn:NearOptimismInMain}; we can expand the second term using \cref{eqn:DecompositionInMain}  where $\pi_k$ is the agent's policy in episode $k$ and $a_{tk}=\pi_{tk}(s_{tk})$ for short. For a generic timestep $t$ we obtain
\begin{align*}
	& \(\Vbar_{tk} - V_{t}^{\pi_k}\)(s_{tk}) \leq \Bigg[ \E_{s' \sim p_t(s_{tk},a_{tk})}\(\Vbar_{t+1,k}-V^{\pi_k}_{t+1}\) (s') \\
	& + \IBE + \| \phi_t(s_{tk},a_{tk}) \|_{\Sigma^{-1}_{tk}} \underbrace{\(\sqrt{k}\IBE + \sqrt{\alpha_{tk}} + \sqrt{\beta_{tk}} \)}_{\approx \widetilde O( \sqrt{k}\IBE +  \sqrt{d_t})} \Bigg].
	\end{align*}
Now write $\E_{s' \sim p_t(s_{tk},a_{tk})}\(\Vbar_{t+1,k}-V^{\pi_k}_{t+1}\) (s') $ as $\(\Vbar_{t+1,k}-V^{\pi_k}_{t+1}\) (s_{t+1,k}) $ plus a martingale term $\dot \zeta_{tk}$ which we ignore for brevity (details in appendix). Induction over $t\in[H]$ and summing over $k\in[K]$ gives $\sum_{k=1}^K \sum_{t=1}^H \(\Vbar_{1k} - V_{1}^{\pi_k}\)(s_{1k})$
\begin{align*}
	& \leq  \sum_{k=1}^K \sum_{t=1}^H \Bigg[ \IBE + \| \phi_t(s_{tk},a_{tk}) \|_{\Sigma^{-1}_{tk}}  \times \widetilde O( \sqrt{k}\IBE +  \sqrt{d_t}) \Bigg].
\end{align*}
Recall $ \sum_{k=1}^K  \| \phi_t(s_{tk},a_{tk}) \|_{\Sigma_{tk}^{-1}} = \widetilde O(\sqrt{d_t K})$ from \cite{Abbasi11}; substituting this concludes.

\section{Conclusion}
We have introduced an algorithm for online exploration with linear approximators under the notion of low-inherent Bellman error with an optimal regret bound with regards to statistical rates and the lack of closedness of the Bellman operator. The construction reveals that a shift to global optimization might be unavoidable with more general linear approximators than prior low-rank work, making  computational tractability harder to achieve. A core idea is that by working directly in the parameter space we enable a linear propagation of the errors (as opposed to exponential) and we limit the complexity of the value function class, which can serve as inspiration to improve the statistical efficiency for other algorithms as well. Finally, a noteworthy contribution is our analysis for misspecified  contextual linear bandit, which explains that a simple modification of a mainstream algorithm is sufficient to handle such setting.

\section*{Acknowledgments}
Andrea Zanette is partially supported by a Total Innovation Fellowship. We thank Alekh Agarwal for pointing our the connection with the low Bellman rank setting. The authors are grateful to the reviewers for their helpful comments.

\bibliography{rl}

\begin{thebibliography}{52}
\providecommand{\natexlab}[1]{#1}
\providecommand{\url}[1]{\texttt{#1}}
\expandafter\ifx\csname urlstyle\endcsname\relax
  \providecommand{\doi}[1]{doi: #1}\else
  \providecommand{\doi}{doi: \begingroup \urlstyle{rm}\Url}\fi

\bibitem[Abbasi-Yadkori et~al.(2011)Abbasi-Yadkori, Pal, and
  Szepesvari]{Abbasi11}
Abbasi-Yadkori, Y., Pal, D., and Szepesvari, C.
\newblock Improved algorithms for linear stochastic bandits.
\newblock In \emph{Advances in Neural Information Processing Systems (NIPS)},
  2011.

\bibitem[Agarwal et~al.(2019)Agarwal, Kakade, and Yang]{agarwal2019optimality}
Agarwal, A., Kakade, S., and Yang, L.~F.
\newblock On the optimality of sparse model-based planning for markov decision
  processes.
\newblock \emph{arXiv preprint arXiv:1906.03804}, 2019.

\bibitem[Auer(2002)]{auer2002using}
Auer, P.
\newblock Using confidence bounds for exploitation-exploration trade-offs.
\newblock \emph{Journal of Machine Learning Research}, 3\penalty0
  (Nov):\penalty0 397--422, 2002.

\bibitem[Azar et~al.(2012)Azar, Munos, and Kappen]{Azar12}
Azar, M., Munos, R., and Kappen, H.~J.
\newblock On the sample complexity of reinforcement learning with a generative
  model.
\newblock In \emph{International Conference on Machine Learning (ICML)}, 2012.

\bibitem[Azar et~al.(2017)Azar, Osband, and Munos]{Azar17}
Azar, M.~G., Osband, I., and Munos, R.
\newblock Minimax regret bounds for reinforcement learning.
\newblock In \emph{International Conference on Machine Learning (ICML)}, 2017.

\bibitem[Baird(1995)]{baird1995residual}
Baird, L.
\newblock Residual algorithms: Reinforcement learning with function
  approximation.
\newblock In \emph{International Conference on Machine Learning (ICML)}. 1995.

\bibitem[Bartlett \& Tewari(2009)Bartlett and Tewari]{Bartlett09}
Bartlett, P.~L. and Tewari, A.
\newblock Regal: A regularization based algorithm for reinforcement learning in
  weakly communicating mdps.
\newblock In \emph{Conference on Uncertainty in Artificial Intelligence (UAI)},
  2009.

\bibitem[Dann et~al.(2018)Dann, Jiang, Krishnamurthy, Agarwal, Langford, and
  Schapire]{dann2018oracle}
Dann, C., Jiang, N., Krishnamurthy, A., Agarwal, A., Langford, J., and
  Schapire, R.~E.
\newblock On oracle-efficient pac rl with rich observations.
\newblock In \emph{Advances in Neural Information Processing Systems (NIPS)},
  pp.\  1429--1439, 2018.

\bibitem[Dann et~al.(2019)Dann, Li, Wei, and Brunskill]{dann2019policy}
Dann, C., Li, L., Wei, W., and Brunskill, E.
\newblock Policy certificates: Towards accountable reinforcement learning.
\newblock In \emph{International Conference on Machine Learning}, pp.\
  1507--1516, 2019.

\bibitem[Du et~al.(2019)Du, Kakade, Wang, and Yang]{du2019good}
Du, S.~S., Kakade, S.~M., Wang, R., and Yang, L.~F.
\newblock Is a good representation sufficient for sample efficient
  reinforcement learning?
\newblock \emph{arXiv preprint arXiv:1910.03016}, 2019.

\bibitem[Efroni et~al.(2019)Efroni, Merlis, Ghavamzadeh, and
  Mannor]{efroni2019tight}
Efroni, Y., Merlis, N., Ghavamzadeh, M., and Mannor, S.
\newblock Tight regret bounds for model-based reinforcement learning with
  greedy policies.
\newblock In \emph{Advances in Neural Information Processing Systems}, 2019.

\bibitem[Fruit et~al.(2018)Fruit, Pirotta, Lazaric, and Ortner]{Fruit18}
Fruit, R., Pirotta, M., Lazaric, A., and Ortner, R.
\newblock Efficient bias-span-constrained exploration-exploitation in
  reinforcement learning.
\newblock https://arxiv.org/abs/1802.04020, 2018.

\bibitem[Ghosh et~al.(2017)Ghosh, Chowdhury, and
  Gopalan]{ghosh2017misspecified}
Ghosh, A., Chowdhury, S.~R., and Gopalan, A.
\newblock Misspecified linear bandits.
\newblock In \emph{Thirty-First AAAI Conference on Artificial Intelligence},
  2017.

\bibitem[Golub \& Van~Loan(2012)Golub and Van~Loan]{golub2012matrix}
Golub, G.~H. and Van~Loan, C.~F.
\newblock \emph{Matrix Computations}.
\newblock JHU Press, 2012.

\bibitem[Gopalan et~al.(2016)Gopalan, Maillard, and Zaki]{gopalan2016low}
Gopalan, A., Maillard, O.-A., and Zaki, M.
\newblock Low-rank bandits with latent mixtures.
\newblock \emph{arXiv preprint arXiv:1609.01508}, 2016.

\bibitem[Jaksch et~al.(2010)Jaksch, Ortner, and Auer]{Jaksch10}
Jaksch, T., Ortner, R., and Auer, P.
\newblock Near-optimal regret bounds for reinforcement learning.
\newblock \emph{Journal of Machine Learning Research}, 2010.

\bibitem[Jiang \& Agarwal(2018)Jiang and Agarwal]{jiang2018open}
Jiang, N. and Agarwal, A.
\newblock Open problem: The dependence of sample complexity lower bounds on
  planning horizon.
\newblock In \emph{Conference on Learning Theory (COLT)}, pp.\  3395--3398,
  2018.

\bibitem[Jiang et~al.(2017)Jiang, Krishnamurthy, Agarwal, Langford, and
  Schapire]{jiang17contextual}
Jiang, N., Krishnamurthy, A., Agarwal, A., Langford, J., and Schapire, R.~E.
\newblock Contextual decision processes with low {B}ellman rank are
  {PAC}-learnable.
\newblock In Precup, D. and Teh, Y.~W. (eds.), \emph{International Conference
  on Machine Learning (ICML)}, volume~70 of \emph{Proceedings of Machine
  Learning Research}, pp.\  1704--1713, International Convention Centre,
  Sydney, Australia, 06--11 Aug 2017. PMLR.
\newblock URL \url{http://proceedings.mlr.press/v70/jiang17c.html}.

\bibitem[Jin et~al.(2018)Jin, Allen-Zhu, Bubeck, and Jordan]{jin2018q}
Jin, C., Allen-Zhu, Z., Bubeck, S., and Jordan, M.~I.
\newblock Is q-learning provably efficient?
\newblock In \emph{Advances in Neural Information Processing Systems}, pp.\
  4863--4873, 2018.

\bibitem[Jin et~al.(2020)Jin, Yang, Wang, and Jordan]{jin2020provably}
Jin, C., Yang, Z., Wang, Z., and Jordan, M.~I.
\newblock Provably efficient reinforcement learning with linear function
  approximation.
\newblock In \emph{Conference on Learning Theory}, 2020.

\bibitem[Kolter(2011)]{kolter2011fixed}
Kolter, J.~Z.
\newblock The fixed points of off-policy td.
\newblock In \emph{Advances in Neural Information Processing Systems (NIPS)},
  pp.\  2169--2177, 2011.

\bibitem[Krishnamurthy et~al.(2016)Krishnamurthy, Agarwal, and
  Langford]{krishnamurthy2016pac}
Krishnamurthy, A., Agarwal, A., and Langford, J.
\newblock Pac reinforcement learning with rich observations.
\newblock In \emph{Advances in Neural Information Processing Systems (NIPS)},
  pp.\  1840--1848, 2016.

\bibitem[Krishnamurthy et~al.(2018)Krishnamurthy, Wu, and
  Syrgkanis]{krishnamurthy2018semiparametric}
Krishnamurthy, A., Wu, S., and Syrgkanis, V.
\newblock Semiparametric contextual bandits.
\newblock In \emph{35th International Conference on Machine Learning, ICML
  2018}, pp.\  4330--4349. International Machine Learning Society (IMLS), 2018.

\bibitem[Lagoudakis \& Parr(2003)Lagoudakis and Parr]{lagoudakis2003least}
Lagoudakis, M.~G. and Parr, R.
\newblock Least-squares policy iteration.
\newblock \emph{Journal of machine learning research}, 4\penalty0
  (Dec):\penalty0 1107--1149, 2003.

\bibitem[Lattimore \& Szepesv{\'a}ri(2020)Lattimore and
  Szepesv{\'a}ri]{lattimore2020bandit}
Lattimore, T. and Szepesv{\'a}ri, C.
\newblock \emph{Bandit Algorithms}.
\newblock Cambridge University Press, 2020.

\bibitem[Lattimore \& Szepesvari(2020)Lattimore and
  Szepesvari]{lattimore2020learning}
Lattimore, T. and Szepesvari, C.
\newblock Learning with good feature representations in bandits and in rl with
  a generative model.
\newblock In \emph{International Conference on Machine Learning (ICML)}, 2020.

\bibitem[Lazaric et~al.(2012)Lazaric, Ghavamzadeh, and
  Munos]{lazaric2012finite}
Lazaric, A., Ghavamzadeh, M., and Munos, R.
\newblock Finite-sample analysis of least-squares policy iteration.
\newblock \emph{Journal of Machine Learning Research}, 13\penalty0
  (Oct):\penalty0 3041--3074, 2012.

\bibitem[Maillard et~al.(2014)Maillard, Mann, and Mannor]{Maillard14}
Maillard, O.-A., Mann, T.~A., and Mannor, S.
\newblock ``how hard is my {MDP}?'' the distribution-norm to the rescue.
\newblock In \emph{Advances in Neural Information Processing Systems (NIPS)},
  2014.

\bibitem[Munos(2005)]{munos2005error}
Munos, R.
\newblock Error bounds for approximate value iteration.
\newblock In \emph{AAAI Conference on Artificial Intelligence (AAAI)}, 2005.

\bibitem[Munos \& Szepesv{\'a}ri(2008)Munos and
  Szepesv{\'a}ri]{munos2008finite}
Munos, R. and Szepesv{\'a}ri, C.
\newblock Finite-time bounds for fitted value iteration.
\newblock \emph{Journal of Machine Learning Research}, 9\penalty0
  (May):\penalty0 815--857, 2008.

\bibitem[Osband \& Van~Roy(2014)Osband and Van~Roy]{osband2014near}
Osband, I. and Van~Roy, B.
\newblock Near-optimal reinforcement learning in factored mdps.
\newblock In \emph{Advances in Neural Information Processing Systems (NIPS)},
  2014.

\bibitem[Puterman(1994)]{puterman1994markov}
Puterman, M.~L.
\newblock \emph{Markov Decision Processes: Discrete Stochastic Dynamic
  Programming}.
\newblock John Wiley \& Sons, Inc., New York, NY, USA, 1994.
\newblock ISBN 0471619779.

\bibitem[Russo(2019)]{russo2019worst}
Russo, D.
\newblock Worst-case regret bounds for exploration via randomized value
  functions.
\newblock In \emph{Advances in Neural Information Processing Systems}, 2019.

\bibitem[Shreve \& Bertsekas(1978)Shreve and Bertsekas]{shreve1978alternative}
Shreve, S.~E. and Bertsekas, D.~P.
\newblock Alternative theoretical frameworks for finite horizon discrete-time
  stochastic optimal control.
\newblock \emph{SIAM Journal on control and optimization}, 16\penalty0
  (6):\penalty0 953--978, 1978.

\bibitem[Sidford et~al.(2018)Sidford, Wang, Wu, Yang, and Ye]{Sidford18}
Sidford, A., Wang, M., Wu, X., Yang, L.~F., and Ye, Y.
\newblock Near-optimal time and sample complexities for for solving discounted
  markov decision process with a generative model.
\newblock In \emph{Advances in Neural Information Processing Systems (NIPS)},
  2018.

\bibitem[Simchowitz \& Jamieson(2019)Simchowitz and
  Jamieson]{simchowitz2019non}
Simchowitz, M. and Jamieson, K.
\newblock Non-asymptotic gap-dependent regret bounds for tabular mdps.
\newblock \emph{arXiv preprint arXiv:1905.03814}, 2019.

\bibitem[Sun et~al.(2018)Sun, Jiang, Krishnamurthy, Agarwal, and
  Langford]{sun2018model}
Sun, W., Jiang, N., Krishnamurthy, A., Agarwal, A., and Langford, J.
\newblock Model-based reinforcement learning in contextual decision processes.
\newblock \emph{arXiv preprint arXiv:1811.08540}, 2018.

\bibitem[Sutton \& Barto(2018)Sutton and Barto]{sutton2018reinforcement}
Sutton, R.~S. and Barto, A.~G.
\newblock \emph{Reinforcement learning: An introduction}.
\newblock MIT Press, 2018.

\bibitem[Tossou et~al.(2019)Tossou, Basu, and Dimitrakakis]{tossou2019near}
Tossou, A., Basu, D., and Dimitrakakis, C.
\newblock Near-optimal optimistic reinforcement learning using empirical
  bernstein inequalities.
\newblock \emph{arXiv preprint arXiv:1905.12425}, 2019.

\bibitem[Tsitsiklis \& Van~Roy(1996)Tsitsiklis and
  Van~Roy]{tsitsiklis1996feature}
Tsitsiklis, J.~N. and Van~Roy, B.
\newblock Feature-based methods for large scale dynamic programming.
\newblock \emph{Machine Learning}, 22\penalty0 (1-3):\penalty0 59--94, 1996.

\bibitem[Van~Roy \& Dong(2019)Van~Roy and Dong]{van2019comments}
Van~Roy, B. and Dong, S.
\newblock Comments on the du-kakade-wang-yang lower bounds.
\newblock \emph{arXiv preprint arXiv:1911.07910}, 2019.

\bibitem[Vershynin(2010)]{vershynin2010introduction}
Vershynin, R.
\newblock Introduction to the non-asymptotic analysis of random matrices.
\newblock \emph{arXiv preprint arXiv:1011.3027}, 2010.

\bibitem[Wainwright(2019)]{wainwright2019high}
Wainwright, M.~J.
\newblock \emph{High-dimensional statistics: A non-asymptotic viewpoint},
  volume~48.
\newblock Cambridge University Press, 2019.

\bibitem[Wang et~al.(2019)Wang, Wang, Du, and Krishnamurthy]{wang2019optimism}
Wang, Y., Wang, R., Du, S.~S., and Krishnamurthy, A.
\newblock Optimism in reinforcement learning with generalized linear function
  approximation.
\newblock \emph{arXiv preprint arXiv:1912.04136}, 2019.

\bibitem[Yang \& Wang(2019)Yang and Wang]{yang2019sample}
Yang, L.~F. and Wang, M.
\newblock Sample-optimal parametric q-learning with linear transition models.
\newblock In \emph{International Conference on Machine Learning (ICML)}, 2019.

\bibitem[Yang \& Wang(2020)Yang and Wang]{yang2020reinforcement}
Yang, L.~F. and Wang, M.
\newblock Reinforcement leaning in feature space: Matrix bandit, kernels, and
  regret bound.
\newblock In \emph{International Conference on Machine Learning (ICML)}, 2020.

\bibitem[Zanette \& Brunskill(2018)Zanette and Brunskill]{Zanette18a}
Zanette, A. and Brunskill, E.
\newblock Problem dependent reinforcement learning bounds which can identify
  bandit structure in mdps.
\newblock In \emph{International Conference on Machine Learning (ICML)}, 2018.

\bibitem[Zanette \& Brunskill(2019)Zanette and Brunskill]{zanette2019tighter}
Zanette, A. and Brunskill, E.
\newblock Tighter problem-dependent regret bounds in reinforcement learning
  without domain knowledge using value function bounds.
\newblock In \emph{International Conference on Machine Learning (ICML)}, 2019.
\newblock URL \url{http://proceedings.mlr.press/v97/zanette19a.html}.

\bibitem[Zanette et~al.(2019{\natexlab{a}})Zanette, Brunskill, and {J.
  Kochenderfer}]{zanette2019b}
Zanette, A., Brunskill, E., and {J. Kochenderfer}, M.
\newblock Almost horizon-free structure-aware best policy identification with a
  generative model.
\newblock In \emph{Advances in Neural Information Processing Systems},
  2019{\natexlab{a}}.

\bibitem[Zanette et~al.(2019{\natexlab{b}})Zanette, Lazaric, {J. Kochenderfer},
  and Brunskill]{zanette19limiting}
Zanette, A., Lazaric, A., {J. Kochenderfer}, M., and Brunskill, E.
\newblock Limiting extrapolation in linear approximate value iteration.
\newblock In \emph{Advances in Neural Information Processing Systems},
  2019{\natexlab{b}}.

\bibitem[Zanette et~al.(2020)Zanette, Brandfonbrener, Pirotta, and
  Lazaric]{zanette2020frequentist}
Zanette, A., Brandfonbrener, D., Pirotta, M., and Lazaric, A.
\newblock Frequentist regret bounds for randomized least-squares value
  iteration.
\newblock In \emph{AISTATS}, 2020.

\bibitem[Zhang \& Ji(2019)Zhang and Ji]{zhang2019regret}
Zhang, Z. and Ji, X.
\newblock Regret minimization for reinforcement learning by evaluating the
  optimal bias function.
\newblock In \emph{Advances in Neural Information Processing Systems}, pp.\
  2827--2836, 2019.

\end{thebibliography}
\bibliographystyle{icml2020}

\onecolumn
\appendix
\newpage 

\section{Symbols}
\renewcommand{\arraystretch}{1.5}
\begin{longtable}{l l l}
\caption{Symbols}\\
\hline
$\Lphi$ & $ \defeq $& upper bound on $\sup_{s,a,t} \| \phi_t(s,a) \|_2$ \\
$\Radius_t$ & $\defeq$ & upper bound on $\|\overline \theta_t \|_2$ for $\overline \theta_t \in \mathcal B_t$, that is $\|\overline \theta_t \|_{2} \leq \Radius_t$ \\
$\Radius$ & $\defeq$ & $\max_{t\in[H]} \Radius_t$ \\
$\mathcal B_t$ & $\defeq$ & set for $\overline \theta_t$ \\
$a\mathcal B_t$ & $\defeq$ & $ \{ ax \mid x \in \mathcal B_t \}$ for a positive real $a$ \\ 
$F_{ij}$ & $\defeq$ & failure events, see \cref{def:FailureEvent} \\
$s_{tk}$ & $\defeq$& state encountered in timestep $t$ of episode $k$ \\
$a_{tk}$ & $\defeq$& action played in timestep $t$ of episode $k$, i.e., $a_{tk} = \pi_{tk}(s_{tk})$ \\
$r_{tk}$ & $\defeq$& reward experienced in timestep $t$ of episode $k$ after playing $a_{tk}$ in $s_{tk}$ \\
$r_{t}(s,a)$ & $\defeq$& average reward at timestep $t$ after playing $a$ in $s$ \\
$p_{t}(s,a)$ & $\defeq$& transition function at timestep $t$ after playing $a$ in $s$ \\
  $\dot \zeta_{tk}$ & $\defeq$ & $\E_{s' \sim p_t(s_{tk},a_{tk})}\(\Vbar_{t+1,k}-\Vpi_{t+1}\) (s') - \(\Vbar_{t+1,k}-\Vpi_{t+1}\) (s_{t+1,k})$ \\
  $\phi_{tk}$ & $\defeq$ & Feature encountered in timestep $t$ of episode $k$, i.e., $\phi_t(s_{tk},a_{tk})$ \\
  $\mathcal V_{t}$ & $\defeq$ & $\{ V \mid V(s) = \max_a \phi_{t}(s,a)^\top \theta, \theta \in \mathcal B_t \}$ \\
  $\mathcal Q_{t}$ & $\defeq$ & $\{ Q \mid Q(s,a) = \phi_{t}(s,a)^\top \theta, \theta \in \mathcal B_t \}$ \\
  $\eta_{ti}(V_{t+1})$ & $\defeq $ & $ r_{ti} - r_t(s_{ti},a_{ti}) + V_{t+1}(s_{t+1,i}) -  \E_{s'\sim p_t(s_{ti},a_{ti})} V_{t+1}(s')$ (this is for a generic $V_{t+1})$ \\
   $\eta_{tki} $ & $\defeq$ & $\eta_{ti}(\Vbar_{t+1,k})$ \\
  $\dot \zeta_{tk}$ & $\defeq$ & $\dotzetadef$ \\
  $\sqrt{\beta_{tk}} $ & $\defeq$ & $\sqrtbetadef$ \\
  $\sqrt{\alpha_{tk}} $ & $\defeq$ & $\sqrtalphadef$  \\
  $\delta'$ & $\defeq$ &$\frac{\delta}{2T}$ \\
  $Q_{t}(\theta)$ & $\defeq$ & function that maps $(s,a)\mapsto \phi_t(s,a)^\top\theta$ \\
  $V_{t}(\theta)$ & $\defeq$ & function that maps $ s \mapsto \max_{a'}\phi_t(s,a)^\top\theta$ \\
  $\T_{t}(Q)$ & $\defeq$ & function $Q^+$ that maps $(s,a)\mapsto r_t(s,a) + \E_{s'\sim p_t(s,a)} \max_{a'}Q(s',a')$ \\
  $\mathring \theta_t(Q) $ & $\defeq$ & $\argmin_{\theta \in \mathcal B_t} \sup_{(s,a)} |\phi_t(s,a)^\top \theta - \(\T_t Q\) (s,a)|$  (ties broken arbitrarily) \\
    $\mathring \Delta_t(Q) $ & $\defeq$ & $\min_{\theta \in \mathcal B_t} \sup_{(s,a)} |\phi_t(s,a)^\top \theta - \(\T_t Q\)(s,a) |$  \\
    $\Sigma_{tk} $ & $\defeq$ & $\sum_{i=1}^{k-1} \phi_{ti}\phi_{ti}^\top + \lambda I$ \\
    $V_t^\pi $ & $\defeq$ & value function of policy $\pi$ at timestep $t$
\label{tab:MainNotation}
\end{longtable}

\newpage
\section{On the Inherent Bellman Error}
If $\IBE = 0$ one could represent $\Qstar$ using a linear representation; in addition, having no inherent Bellman error is equivalent to having linear rewards with transitions to elements of $\mathcal V_{t+1}$ that appears to be linear. 
For simplicity, the discussion is with $\mathcal B_t = \R^{d_t}$, though this is not the only possible choice.
\linearity
\begin{proof}
Since the zero vector $0\in \mathcal Q_{t}$ (by construction, otherwise $\mathcal B_t = \emptyset$) at all timesteps, 
	for any $t \in [H]$ we certainly have (by choosing $0 = Q_{t+1} \in \mathcal Q_{t+1}$ in the outer $\sup$ of \cref{def:InherentBellmanError}):
	\begin{align}
			0 = \inf_{\theta_{t} \in \mathcal B_{t}} \sup_{(s,a)\in \StateSpace\times\ActionSpace}|\phi_t(s,a)^\top \theta_t - \(\T_t(0)\)(s,a)| = \inf_{\theta_{t} \in \mathcal B_{t}} \sup_{(s,a)\in \StateSpace\times\ActionSpace}|\phi_t(s,a)^\top \theta_t - r_t(s,a)| 
	\end{align}

Now, for the second part of the proof,
\begin{align}
	0 =  \sup_{\theta_{t+1} \in \mathcal B_{t+1}} \inf_{\theta_{t} \in \mathcal B_{t}} \sup_{(s,a)\in \StateSpace\times\ActionSpace}|\phi_t(s,a)^\top \theta_t - \(r_t(s,a) + \E_{s'\sim p_t(s,a)} V_{t+1}(\theta_{t+1})(s')\)|.
\end{align}
Using the just reward linearity just shown:
\begin{align}
	0 = \sup_{\theta_{t+1} \in \mathcal B_{t+1}} \inf_{\theta_{t} \in \mathcal B_{t}} \sup_{(s,a)\in \StateSpace\times\ActionSpace}|\phi_t(s,a)^\top \( \theta_t - \theta_t^R\) - \E_{s'\sim p_t(s,a)} V_{t+1}(\theta_{t+1})(s')|.
\end{align}
Since $\theta_t^R\in\mathcal B_t$, we certainly have $\theta^P_t\defeq\( \theta_t - \theta_t^R\) \in \mathcal B_t$. 
\end{proof}

Next we examine the relation between low rank MDPs and MDPs with no inherent Bellman error. One direction of the following proposition also appeared in \cite{yang2019sample} (proposition 2).
We recall that a measure $\psi_t$ is a positive function with $\|\psi_t(\cdot)\|_{TV} = 1$.

\LowrankVsNOIBE
\begin{proof}
$(\Rightarrow)$\\
Assume the MDP is low rank in the sense of \cref{eqn:LinearMDPequations}.
	Let $\theta_{t+1}\in\mathcal B_{t+1}$. Then 
	\begin{align}
		\T_t(Q_{t+1}(\theta_{t+1}))(s,a) & = \phi_t(s,a)^\top\theta^R_t + \int_{s'\in\StateSpace}\phi_t(s,a)\psi_t(s')V_{t+1}(\theta_{t+1})(s') ds' \\
		& = \phi_t(s,a)^\top 
\underbrace{\( \theta_t^R + \int_{s'\in\StateSpace}\psi_t(s')V_{t+1}(\theta_{t+1})(s') ds' \)}_{\defeq \theta_{t} \in \mathcal B_t}.
	\end{align}
Thus $\IBE = 0$. \\
$(\Leftarrow)$\\
Fix $N \in \mathbb N^+$ and consider the chain with a starting state in the middle ($s = 0$) with $N$ states to the left and $N$ to the right. The agent can go one unit to the right or one to the left in each timestep by choosing action $+1$ or $-1$, respectively, or stay put by choosing action $a = 0$. 
The total time available within an episode is $H=N+1$, and there is a reward in the leftmost state and a reward in the rightmost state, zero everywhere else. Formally:
\begin{itemize}
\item $\StateSpace  = \{-N-1,\dots,N+1 \}$
\item $\ActionSpace = \{-1,0,+1\}$
\item $H = N+1$
\item $p_t(s,a) = e_{s+a}$ (here $e_i$ is the canonical vector with a one in the $i$-th position and zero otherwise)
\item $r_H[H,1] = r^\star_{+1}$, $r_H[-H,-1] = r^\star_{-1}$, and $0$ otherwise, with $r^\star_{+1} \in \R$, $r^\star_{-1} \in \R$.	
\end{itemize}

Clearly the transition matrix is not low rank (in the sense of being independent of $N$), for any choice of the feature representation. For example for the policy $\pi_t(s) = 0$ we have that $P^\pi = I$, which is full rank. Now consider the feature representation:
\begin{align}
\phi_t(s,a)=
\begin{cases}
[1,0], \; \text{if} \; (s,a) = (+t,+1) \\
[0,1], \; \text{if} \; (s,a) = (-t,-1) \\
[0,0], \; \text{otherwise}.
\end{cases}
\end{align}
The feature dimensionality is $ = 2 \neq N$, so this is not a low-rank MDP according to equation \cref{eqn:LinearMDPequations}.

We claim that this gives $0$ inherent Bellman error. Indeed, it's easy to verify this by inspection, $|s_t| = t-1$ are the only two reachable states at timestep $t$ with at least an action with non-zero feature:
\begin{align}
	\forall \theta_{t+1}\quad  \exists \theta_t^+  \quad  \textrm{such that} \quad \|  Q_{t}(\theta^+_{t})  -  \T_t Q_{t+1}(\theta_{t+1}) \|_\infty = 0
\end{align}
In particular, set
$\theta^+_{t} = [\max\{0, \theta_{t+1}[1]\},\max\{0, \theta_{t+1}[2]\}]$ for $t = 1,\dots,H-1$ and $\theta^+_H  \defeq [r^\star_{+1},r^\star_{-1}]$.
\end{proof}
The next step is to show that, likewise, low-rank MDPs imply that every policy has a linearly parameterizable action-value function, but not viceversa. The first direction is established by, for example, proposition 2.3 in \cite{jin2020provably}.

\LowrankVsLP
\begin{proof}
$(\Rightarrow)$\\
Assume by induction that $Q_{t+1}^\pi\in \mathcal Q_{t+1}$,
and proceed as the first part of the proof of \cref{prop:LowrankVsNOIBE} 
(but with the Bellman operator of policy $\pi$ (as $\T_t^\pi$) in place of $\T_t$) 
to conclude $\theta_t \in \mathcal B_t$, showing the inductive step. The base case is immediate. \\
$(\Leftarrow)$\\
Now, for the viceversa not being true, consider the same MDP as in the proof of \cref{prop:LowrankVsNOIBE}; as already shown, this is not a low-rank MDP. On the other hand, the policies can be in three disjoint sets (we adopt the same feature representation as in the proof of \cref{prop:LowrankVsNOIBE}): for $|s| \leq t - 1 $ (we cannot reach states outside of this range at timestep $t$) we can write

\textbf{1) Policies that always go right}
We have $Q_t^{\pi}(s,a) = \phi_t(s,a)^\top [r^\star_{+1},0]$ (by inspection)

\textbf{2) Policies that always go left}
We have $Q_t^{\pi}(s,a) = \phi_t(s,a)^\top [0,r^\star_{-1}]$ (by inspection)

\textbf{3) All other policies}
We have $Q_t^{\pi}(s,a) = \phi_t(s,a)^\top [0,0]$ (by inspection) \\

In other words, we can represent the cumulated return of each policy.
The proof is complete, since the MDP is not low rank with this feature representation.
\end{proof}

Finally, we compare MDPs with linear architectures which have $\IBE = 0$ with those where every policy has an action-value function linearly parameterizable. As we show next, these are quite different assumptions, although an intersection is possible by combining the proofs of the prior two propositions.

\NOIBEvsLP
This suggests that, depending on the parameterization, different algorithms may be preferable for solving the MDP (i.e., finding the optimal policy). In particular, if $\IBE = 0$ then approximate value iteration converges to the global optimum; viceversa, if all policies are linearly parameterizable then approximate policy improvement should be used.
\begin{proof}
$(\Rightarrow)$\\
Consider an MDP with two groups ($A$ and $B$) of non-communicating states, i.e., with states $s^A_{1},\dots,s^A_{H}$ and $s^B_{1},\dots,s^B_{H}$. The starting state is either $s_1^A$ or $s_1^B$. There is only one action except in $s^A_{H},s^B_{H}$. From state $s^A_{i}$ the transition to $s^A_{i+1}$ is deterministic as long as $i \in [H-1]$ and likewise from $s^B_{i}$ to  $s^B_{i+1}$. In $s^A_{H}$ and $s^B_{H}$ there are two actions with identical outcome regardless of the state. In particular, both $(s^A_H,0)$ and $(s^B_H,0)$ give a return of $0$ while both $(s^A_H,1)$ and $(s^B_H,1)$  give a return of $1$; both terminate the episode.

Let the parameterization be $\phi_t(\cdot,\cdot) = 1$ for any state indexed $<H$, for the only available action. In the last timestep, $\phi_H(s^A_H,0) = \phi_H(s^B_H,0) = 0$ and $\phi_H(s^A_H,1) = \phi_H(s^B_H,1) = 1$. 

It's easy to see (by inspection) that this MDP has $\IBE = 0$: in any timestep $t\in[H-1]$ we have $\Qbar_t(s^A_{t},\cdot) = \Qbar_t(s^B_t,\cdot) = \Vbar_{t+1}(s^B_{t+1}) = \Vbar_{t+1}(s^A_{t+1})$ by using an identical parameter $ \theta_t = \theta_{t+1}$ (notice that there is only one action for $t \in [H-1]$). 
In other words, $\forall \theta_{t+1}, \exists \theta_{t} (= \theta_{t+1})$ that gives $Q_t(\cdot,\cdot) = \T_tQ_{t+1}(\cdot,\cdot)$ with $Q_t \in \mathcal Q_t$ and $Q_{t+1}\in \mathcal Q_{t+1}$ for all reachable states at timestep $t\in[H-1]$. Finally, the last timestep can be expressed as linear bandit problem. Thus $\IBE = 0 $. 

However, consider policy $\pi^{x}$ that takes two different actions in the last states, i.e., $\pi^x_{H}(s_H^A) = 1 \neq 0 = \pi^x_{H}(s_H^B) $. The return of the policies differs, indicating that for any $t \in [H-1]$, $Q_t^{\pi^x}(s^A_t,\cdot) \neq Q_t^{\pi^x}(s^B_t,,\cdot)$, but our parameterization forces $Q_t(s^A_t,\cdot) = Q_t(s^B_t,\cdot)$ if $Q_t \in \mathcal Q_t$, and therefore the policies do not have an action-value function that is linearly parameterizable.

$(\Leftarrow)$\\
 (Construction inspired by the linear bandit example in \cite{zanette19limiting})
Consider a chain mdp with states $s_{1},\dots,s_{H}$, and starting state $s_1$. \emph{Any} action deterministically leads to the next state, i.e., from $s_i$ to $s_{i+1}$, for $i\in[H-1]$, and does not yield any reward. There are two actions in each state with associated feature $\phi_t(\cdot,-1) = -1$ and $\phi_t(\cdot,+1) = +1$. In particular, notice that the approximator cannot represent the same value for different actions since  $Q_t(\theta_{t})(s_t,+1) = - Q_t(\theta_{t})(s_t,-1)$ must hold by construction.

Since there is no reward in the MDP, every policy has zero return for any state-action at any intermediate timestep, so $Q_t^\pi(s,a) = \phi_t(s,a)^\top \theta^\pi_t$ with $\theta^\pi_t = 0$ certainly holds at any $(s,a,t)$ triplet.
 Yet, for example, for $ \theta_{t+1} = 1$, the corresponding value function is (in the only possible state $s_{t+1}$) $V_{t+1}(\theta_{t+1})(s_{t+1}) = \max_{a} \phi_{t+1}(s_{t+1},a)^\top \theta_{t+1} = 1$. Quite clearly, $\( \T_t V_{t+1}(\theta_{t+1})\) (s_t,\cdot) = V_{t+1}(\theta_{t+1})(s_{t+1})$ (i.e., the value function stays constant since there are no rewards) since there is no rewards in the system and the transition is the same for both actions. However, the approximator cannot represent the same value for different actions since they use opposite (in sign) features, i.e., $Q_t(\theta_{t})(s_t,+1) = - Q_t(\theta_{t})(s_t,-1)$ must hold by construction, which means the inherent Bellman error is strictly positive. 
\end{proof}
The above construction uses an MDP with zero reward function for the sake of clarity of exposition; it is possible to augment the MDP in an obvious way to include rewards by including a ``fork'' at the beginning, similarly to \cite{zanette19limiting}.

\newpage 
\section{\Alg{}}
\subsection{First Step Analysis}

\begin{lemma}[First Step Analysis]
\label{lem:FirstStepAnalysis}
If the program of \cref{def:PlanningOptimizationProgram} admits a feasible solution then the $\overline \theta_{t}$'s must satisfy for $t \in [H]$:
\begin{align}
\overline \theta_{t} = \RecursionExpansion
\end{align}
Furthermore, outside of the failure event of \cref{def:FailureEvent} it holds that:
\begin{align}
|\( \overline Q_{t} (s,a) - \T_t\Qbar_{t+1}\)(s,a) | \leq \IBE + \|\phi_t(s,a)\|_{\Sigma^{-1}_{tk}}\( \sqrt{k}\IBE + \sqrt{\alpha_{tk}} + \sqrt{\beta_{tk}} + \sqrt{\lambda} \Radius_t\).
\end{align} 

\end{lemma}
\begin{proof}
We start by recalling (see constraint of the program of  \cref{def:PlanningOptimizationProgram}):
\begin{align}
\overline \theta_{t} \defeq \overline \xi_{t} + \widehat \theta_{t}.
\end{align}
Now we use the fact that $\widehat \theta_{t}$ must satisfy its constraint written in the program of  \cref{def:PlanningOptimizationProgram}, where $\Vbar_{t+1}(s') = \max_{a'}\Qbar_{t+1}(s',a')$ and $\Qbar_{t+1}(s',a') = \phi_{t+1}(s',a')^\top \overline \theta_{t+1}$:
\begin{align*}
		& =  \overline \xi_{t} + \( \sum_{t=1}^{k-1}\phi_{ti}\phi_{ti}^\top + \lambda I \)^{-1} \sum_{i=1}^{k-1}\phi_{ti}\Big[ r_{ti} + \Vbar_{t+1}(s_{t+1,i})  \Big] \\
		& = \overline \xi_{t} + \( \sum_{t=1}^{k-1}\phi_{ti}\phi_{ti}^\top + \lambda I \)^{-1} \sum_{i=1}^{k-1}\phi_{ti}\Big[ r_t(s_{ti},a_{ti}) + \E_{s'\sim p_t(s_{ti},a_{ti})}\Vbar_{t+1,k}(s') + \eta_{ti}(\Vbar_{t+1}) \Big]
\numberthis{\label{eqn:ExpansionBegin}}
\end{align*}
where in particular, 
\begin{align}
\eta_{ti}(\Vbar_{t+1}) \defeq r_{ti} - r_t(s_{ti},a_{ti}) + \Vbar_{t+1}(s_{t+1,i}) -  \E_{s'\sim p_t(s_{tk},a_{tk})} \Vbar_{t+1}(s').
\end{align}
Recall the following definition of Bellman operator:
\begin{align}
\label{eqn:TQ1}
\( \T_t \Qbar_{t+1}\)(s_{ti},a_{ti}) \defeq r_t(s_{ti},a_{ti})  + \E_{s'\sim p_t(s_{ti},a_{ti})}\max_{a'}\Qbar_{t+1}(s',a').
\end{align}
The key step is now the following: by construction, if a solution to the program of \cref{def:PlanningOptimizationProgram} exists, then in particular  $(\overline \theta_{1},\dots,\overline \theta_{H})$ must satisfy the ball constraint $\overline \theta_{t} \in \mathcal B_t$ for all $t\in[H]$ which implies that each $\Qbar_t$ function belongs to the prescribed functional space $\mathcal Q_t$.
With this in mind, denote with $\mathring \theta_t(\Qbar_{t+1})$ the parameter $\in \mathcal B_t$ that best approximates the Bellman backup of $\Qbar_{t+1}$
and with $\mathring \Delta_t(\Qbar_{t+1})$ the ``residual'' function, see  \cref{tab:MainNotation}.
This allows us to use the value of the finite inherent Bellman error of  \cref{def:InherentBellmanError} to write:
\begin{align}
\label{eqn:eqn15}
	\( \T_t \Qbar_{t+1}\)(s,a) = \phi_t(s,a)^\top\mathring \theta_t(\Qbar_{t+1})+\mathring \Delta_t(\Qbar_{t+1})(s,a).
\end{align}
Comparing the above display (with $(s,a) = (s_{ti},a_{ti})$) against \cref{eqn:TQ1} and then plugging back into \cref{eqn:ExpansionBegin} and using the definition of $\Sigma^{-1}_{tk}$ we can write:
\begin{align*}
		& = \overline \xi_{t} + \(\sum_{i=1}^{k-1}\phi_{ti}\phi_{ti}^\top  + \lambda I \)^{-1}\(\sum_{i=1}^{k-1}\phi_{ti}^\top\overbrace{\(\phi_{ti}\mathring \theta_t(\Qbar_{t+1}) + 
	\mathring \Delta_{t}(\Qbar_{t+1})(s_{ti},a_{ti}) \)}^{=\T_t(\Qbar_{t+1})(s_{ti},a_{ti})} + \lambda \mathring \theta_t(\Qbar_{t+1}) - \lambda\mathring \theta_t(\Qbar_{t+1}) \)  + \Sigma^{-1}_{tk}\sum_{i=1}^{k-1} \phi_{ti} \eta_{ti}(\Vbar_{t+1}) \\
 		& = \RecursionExpansion	
\numberthis{\label{eqn:middledisplay}}
\end{align*}
	This proves the first part of the lemma. 
	
	To show the second part, premultiply the above display by $\phi_t(s,a)^\top$; the left hand side becomes $\Qbar_t(s,a)$ by definition and we proceed to bound each term of the rhs.
	First, \cref{eqn:eqn15} allows us to write:
		\begin{align}
		\label{eqn:17}
		\phi_t(s,a)^\top\mathring \theta_t(\Qbar_{t+1}) \defeq \(\T_t\Qbar_{t+1}\)(s,a) - \mathring \Delta_t(\Qbar_{t+1})(s,a)
	\end{align}
	with $|\mathring \Delta_t(\Qbar_{t+1})(s,a)|
\leq \IBE$.	Cauchy-Schwartz and then \cref{lem:ProjectionBound} give:
	\begin{align}
	\Big| \phi_t(s,a)^\top\Sigma^{-1}_{tk}\sum_{i=1}^{k-1} \phi_{ti}^\top 
	\mathring \Delta_{t}(\Qbar_{t+1})(s_{ti},a_{ti}) \Big| \leq \| \phi_t(s,a) \|_{\Sigma^{-1}_{tk}} \|\sum_{i=1}^{k-1} \phi_{ti}^\top 
	\mathring \Delta_{t}(\Qbar_{t+1})(s_{ti},a_{ti}) \|_{\Sigma^{-1}_{tk}}	\leq  \| \phi_t(s,a) \|_{\Sigma^{-1}_{tk}} \sqrt{k}\IBE.
	\end{align}
Again Cauchy-Schwartz as done above allows us to write (outside of the failure event):
\begin{align}
	\Big|\phi_t(s,a)^\top \Sigma^{-1}_{tk}\sum_{i=1}^{k-1} \phi_{ti} \eta_{ti}(\Vbar_{t+1}) \Big|& \leq \sqrt{\beta_{tk}}\|\phi_t(s,a)\|_{\Sigma^{-1}_{tk}}.
	\end{align}
Cauchy-Schwartz applied to the term below also gives (by definition / constraints on $\overline \xi_{t}$):
\begin{align}
	\Big|\phi_t(s,a)^\top \overline \xi_{t}\Big|& \leq \sqrt{\alpha_{tk}}\|\phi_t(s,a)\|_{\Sigma^{-1}_{tk}}.
\end{align}
Finally, Cauchy-Schwartz with \cref{lem:WorstCaseBound} gives (since $\mathring \theta_t(\Qbar_{t+1}) \in \mathcal B_t$ ):
\begin{align}
	\Big|\phi_t(s,a)^\top \lambda\Sigma^{-1}_{tk}\mathring \theta_t(\Qbar_{t+1}) \Big| & \leq \lambda \| \phi_t(s,a)\|_{\Sigma^{-1}_{tk}} \|\mathring \theta_t(\Qbar_{t+1}) \|_{\Sigma^{-1}_{tk}}  \leq \sqrt{\lambda} \Radius_t\|\phi_t(s,a)\|_{\Sigma^{-1}_{tk}}.
\end{align}
Plugging the bounds back gives the thesis.
\end{proof}

\newpage
\subsection{Failure Event and their Probabilities}
In this section we introduce the failure modes of the algorithm. Whenever a failure event occurs, we cannot guarantee the overall performance of the algorithm. 
\begin{definition}[Failure Events]
\label{def:FailureEvent}
	We define the following failure event in episode $k$:
	\begin{align}
	F_{tk} & \defeq \Bigg\{ \exists V_{t+1} \in \mathcal V_{t+1} \quad \textrm {such that} \quad \Big\|\sum_{i=1}^{k-1} \phi_{ti} \(r_{ti} - r_t(s_{ti},a_{ti}) + V_{t+1}(s_{t+1,i}) - \E_{s'\sim p_t(s_{ti},a_{ti})} V_{t+1}(s') \) \Big\|_{\Sigma^{-1}_{tk}} > \sqrt{\beta_{tk}}	\Bigg\}.
	\end{align}
We call failure event in episode $k$ the union of these events over the within-episode timestep $t\in[H]$:
\begin{align}
F_{k} \defeq \bigcup_{t \in [H]} F_{tk},
\end{align}
and failure event of the algorithm the union of the above events over all the episodes:
\begin{align}
F \defeq \bigcup_{k \in [K]} F_{k}.
\end{align}

\end{definition}

\begin{lemma}[Total Failure Probability]
\label{lem:FailureEventProbability}
Under assumption \ref{ass:MainAssumption} it holds that:
\begin{align}
\Pro\( F \) \leq \frac{\delta}{2}, \quad \forall k \in [K].
\end{align}
\end{lemma}
\begin{proof}
By union bound:	
\begin{align}
 \Pro\(F\) & \defeq \Pro\(  \bigcup_{k \in [K]}  \bigcup_{t \in [H]} F_{tk} \) \\
 & \leq \sum_{k=1}^K \sum_{t=1}^H \Pro\( F_{tk} \) \\
 & \leq  T \delta'.
\end{align}
The last step is from \cref{lem:TransitionsHPbound};
the thesis follows by setting $\delta' = \frac{\delta}{2T}$.
\end{proof}

\begin{lemma}[Transition Noise High Probability Bound]
\label{lem:TransitionsHPbound}
If $\lambda = 1$, with probability at least $1-\delta'$ for all $V_{t+1}\in\mathcal V_{t+1}$ it holds that
\begin{align}
	\Big\|\sum_{i=1}^{k-1} \phi_{ti} \(r_{ti} - r_t(s_{ti},a_{ti}) + V_{t+1}(s_{t+1,i}) - \E_{s'\sim p_t(s_{ti},a_{ti})} V_{t+1}(s') \) \Big\|_{\Sigma^{-1}_{tk}} \leq \sqrt{\beta_{tk}}
	\end{align}
	where:
	\begin{align}
	\sqrt{\beta_{tk}} \defeq \sqrtbetadef.	
	\end{align}
\end{lemma}
\begin{proof}
We start by constructing an $\epsilon$-cover for the set $\mathcal V_{t+1}$ using the supremum distance. To achieve this, we construct an $\epsilon$-cover for the parameter $\theta_{t+1} \in \mathcal B_{t+1}$ using \fullref{lem:CoveringNumberOfEuclideanBall}. This ensures that there exists a set $\mathcal D_{t+1} \subseteq \mathcal B_{t+1}$, containing $(1+2\Radius/\epsilon')^{d_{t+1}}$ vectors $\overset{\triangle}{\theta}_{t+1}$ that well approximates any $\theta_{t+1}\in\mathcal B_{t+1}$:
\begin{align}
\exists \mathcal D_{t+1}\subseteq \mathcal B_{t+1}  \quad \textrm{such that} \quad \forall \theta_{t+1} \in \mathcal B_{t+1}, \quad \exists \overset{\triangle}{\theta}_{t+1} \in \mathcal D_{t+1} \quad \textrm{such that} \quad \| \theta_{t+1} - \overset{\triangle}{\theta}_{t+1} \|_2 \leq \epsilon'.
\end{align}
Let $\overset{\triangle}{V}_{t+1}(s) \defeq \max_{a} \phi_{t+1}(s,a)^\top \overset{\triangle}{\theta}_{t+1}$, where $\overset{\triangle}{\theta}_{t+1} = \argmin_{\overset{\triangle}{\theta}_{t+1} \in \mathcal D_{t+1}} \| \overset{\triangle}{\theta}_{t+1} - \theta_{t+1} \|_2 $.
For any fixed $s \in \StateSpace$ we have that:
\begin{align*}
|\big(V_{t+1} - \overset{\triangle}{V}_{t+1} \big)(s) |& = |\max_{a'} \phi_{t+1}(s,a')^\top \theta_{t+1} - \max_{a''} \phi_{t+1}(s,a'') \overset{\triangle}{\theta}_{t+1} | \\
& \leq | \max_{a} \phi_t(s,a)^\top\big( \theta_{t+1} - \overset{\triangle}{\theta}_{t+1} \big) | \\
& \leq \max_a \| \phi_{t+1}(s,a)\|_2\| \theta_{t+1} - \overset{\triangle}{\theta}_{t+1} \|_2 \\
& \leq \Lphi \epsilon'.
\numberthis{\label{eqn:coveringargument}}
\end{align*}
By using the triangle inequality we can write:
\begin{align*}
	& \Big\|\sum_{i=1}^{k-1} \phi_{ti} \( r_{ti} - r_t(s_{ti},a_{ti}) + V_{t+1}(s_{t+1,k}) - \E_{s'\sim p_t(s_{ti},a_{ti})} V_{t+1}(s') \) \Big\|_{\Sigma^{-1}_{tk}} & \\
	& \leq \Big\|\sum_{i=1}^{k-1} \phi_{ti} \( r_{ti} - r_t(s_{ti},a_{ti}) + \overset{\triangle}{V}_{t+1}(s_{t+1,k}) - \E_{s'\sim p_t(s_{ti},a_{ti})} \overset{\triangle}{V}_{t+1}(s') \) \Big\|_{\Sigma^{-1}_{tk}}+  \\
	& + \Big\| \sum_{i=1}^{k-1} \phi_{ti} \(  \E_{s'\sim p_t(s_{ti},a_{ti})} \overset{\triangle}{V}_{t+1}(s') -  \E_{s'\sim p_t(s_{ti},a_{ti})} V_{t+1}(s') \) \Big\|_{\Sigma^{-1}_{tk}} \\
	& + \Big\| \sum_{i=1}^{k-1} \phi_{ti} \( V_{t+1}(s_{t+1,i}) - \overset{\triangle}{V}_{t+1}(s_{t+1,i}) \) \Big\|_{\Sigma^{-1}_{tk}}.
\numberthis{\label{eqn:Equation1}}
\end{align*}
Each of the last two terms above can be written for some $b_i$'s (different for each of the two terms) as $ \Big\| \sum_{i=1}^{k-1} \phi_{ti} b_i \Big\|_{\Sigma^{-1}_{tk}}$. The projection lemma, \cref{lem:ProjectionBound} ensures (here we are using \cref{eqn:coveringargument} to bound the $b_i$'s):
\begin{align}
	 \Big\| \sum_{i=1}^{k-1} \phi_{ti} b_i \Big\|_{\Sigma^{-1}_{tk}} \leq \Lphi \epsilon' \sqrt{k}.
\end{align}

Now we proceed to compute the probability of the event in the theorem statement.

Denote with $C$ the event reported below, which is a large deviation bound on the first term on the rhs of \fullref{eqn:Equation1}.
\begin{align}
C(\overset{\triangle}{\theta}_{t+1})  \defeq \Bigg\{ \Big\|\sum_{i=1}^{k-1} \phi_{ti} \( r_{ti} - r_t(s_{ti},a_{ti}) + \overset{\triangle}{V}_{t+1}(s_{t+1,i}) - \E_{s'\sim p_t(s_{ti},a_{ti})} \overset{\triangle}{V}_{t+1}(s') \) \Big\|^2_{\Sigma^{-1}_{tk}} > 2\times(1)^2\ln\(\frac{\det(\Sigma_{tk})^{\frac{1}{2}}\det\( \lambda I\)^{-\frac{1}{2}}}{\delta''} \) \Bigg\}.
\end{align}

We obtain that:
\begin{align}
\Pro\Bigg( \bigcup_{\overset{\triangle}{\theta}_{t+1} \in \mathcal D_{t+1}} C(\overset{\triangle}{\theta}_{t+1}) \Bigg) \leq \sum_{\overset{\triangle}{\theta}_{t+1} \in \mathcal D_{t+1}}\Pro\Bigg(  C(\overset{\triangle}{\theta}_{t+1}) \Bigg) \leq  (1+2\Radius_{t+1}/\epsilon')^{d_{t+1}}\delta'' \defeq \delta'
\end{align}
where the last inequality above follows from Theorem 1 in \cite{Abbasi11} with $R = 1$ (the reward and transitions are $1$-subgaussian by assumption \ref{ass:MainAssumption}).
In particular, we set
\begin{align}
\delta'' = \frac{\delta'}{(1+2\Radius_{t+1}/\epsilon')^{d_{t+1}}}	
\end{align}
from the prior display and so with probability $1-\delta'$ we have upper bounded \fullref{eqn:Equation1} by:
\begin{align}
 \sqrt{2\ln\(\frac{\det(\Sigma_{tk})^{\frac{1}{2}}\det\( \lambda I\)^{-\frac{1}{2}}(1+2\Radius_{t+1}/\epsilon')^{d_{t+1}}}{\delta'} \)} + 2\Lphi\epsilon'\sqrt{k}.	
\end{align}
If we now pick
\begin{align}
\epsilon' = \frac{1}{2\Lphi\sqrt{k}}
\end{align}
we get:
\begin{align}
\sqrt{2\ln\(\frac{\det(\Sigma_{tk})^{\frac{1}{2}}\lambda^{-\frac{d_t}{2}}(1+2\Radius_{t+1}/\epsilon')^{d_{t+1}}}{\delta'} \)}+1  = \sqrt{2}\sqrt{\frac{1}{2}\ln\(\det(\Sigma_{tk})\) -\frac{d_t}{2}\ln\(\lambda\) + d_{t+1}\ln(1+2\Radius_{t+1}/\epsilon') + \ln\(\frac{1}{\delta'} \)} + 1
\end{align}
Finally, by setting $\lambda = 1$ and using the Determinant-Trace Inequality (see lemma 10 of \cite{Abbasi11}) we obtain $\det(\Sigma_{tk}) \leq \( 1 +  \Lphi^2 k/d_t\)^{d_t}$
\begin{align}
	\leq \sqrtbetadef \defeq \sqrt{\beta_{tk}}.
\end{align}

\end{proof}

\begin{lemma}[Covering Number of Euclidean Ball]
\label{lem:CoveringNumberOfEuclideanBall}
For any $\epsilon > 0$, the $\epsilon$-covering number of the Euclidean ball
$\R^d$ with radius $R>0$ is upper bounded by $(1+2R/\epsilon)^d$.
\end{lemma}
\begin{proof}
See for example Lemma 5.2 in \cite{vershynin2010introduction}.	
\end{proof}

Finally, the following martingale concentration inequality is well known and will be used later when bounding the regret.
\begin{lemma}[Azuma-Hoeffding Inequality]
\label{prop:Azuma}
Let $X_i$ be a martingale difference sequence such that $X_i \in [-A,A]$ for some $A>0$. Then with probability at least $1-\delta'$ it holds that:
\begin{align}
 \Big| \sum_{i=1}^{n} X_i \Big| \leq \sqrt{2A^2n\ln\( \frac{1}{\delta'} \)}.
\end{align}
\end{lemma}
\begin{proof}
Tha Azuma inequality reads:
\begin{align}
\Pro\(\Big| \sum_{i=1}^{n} X_i \Big| \geq t \) \leq e^{-\frac{2t^2}{4A^2n}},	
\end{align}
see for example \cite{wainwright2019high}.	From here setting the rhs equal to $\delta'$ gives:
\begin{align}
	t \defeq \sqrt{2A^2n\ln\( \frac{1}{\delta'} \)}.
\end{align}
\end{proof}

\newpage
\subsection{Best Approximant and its Properties}
In this section we introduce the $\thetastar$'s parameters, which is the ``best'' sequence of parameters that 1) well approximate the $\Qstar$ values while 2) they satisfy $\thetastar_t \in \mathcal B_t$, so they are going to be a feasible solution for the program of  \cref{def:PlanningOptimizationProgram}, as we show in next section. The $\thetastar$ is not the best parameter that approximates $\Qstar$ (though it's a good enough parameter); rather it's the parameter that one would obtain upon running LSVI in the limit of infinite data and using a minimization of the residual in the $\infty$-norm. 

\begin{definition}[Best Running Approximant in $\infty$-norm]
\label{def:BestApproximant}
We recursively define the best approximant parameter $\thetastar_t$ for $t \in [H]$ as:
\begin{align}
	\thetastar_{t} & \defeq \argmin_{\theta \in \mathcal B_t}  \sup_{(s,a)} \Big| \phi_t(s,a)^\top \theta - \( \T_t Q_{t+1}(\thetastar_{t+1})\)(s,a) \Big|
	\end{align}
with ties broken arbitrarily and $	\thetastar_{H+1} = 0$.
\end{definition}

Using the above definition, we first compute an absolute bound for $|\Qstar_t(s,a) - \phi_t(s,a)^\top\thetastar_t|$ and then use this result to compute the performance bound $\(\Vstar_1 - V_1^{\pi}\)(x_{1})$ from an arbitrary starting state $x_1$ using the policy that can be extracted from $\thetastar$.

\begin{lemma}[Accuracy Bound of $\thetastar$]
\label{lem:AccuracyBoundThetastar}
It holds that:
\begin{align}
\sup_{(s,a)}|\Qstar_t(s,a) - \phi_t(s,a)^\top\thetastar_t| \leq (H-t+1)\IBE.	
\end{align}
\end{lemma}
\begin{proof}
We proceed by induction. Assume that $\sup_{(s,a)}|\Qstar_{t+1}(s,a) - \phi_{t+1}(s,a)^\top\thetastar_{t+1}|  \leq (H-t)\IBE$ for a certain timestep $t+1$ (this is certainly true for $t+1=H+1$). Now consider timestep $t$; the triangle inequality gives us:
\begin{align*}
\sup_{(s,a)}|\Qstar_t(s,a) - \phi_t(s,a)^\top\thetastar_t| & = \sup_{(s,a)}|\( \T_t\Qstar_{t+1}\)(s,a) - \( \T_t Q_{t+1}(\thetastar_{t+1})\)(s,a)  + \( \T_t Q_{t+1}(\thetastar_{t+1})\)(s,a)  - \phi_t(s,a)^\top\thetastar_t| \\
&  \leq \sup_{(s,a)}|\( \T_t\Qstar_{t+1}\)(s,a) - \( \T_t Q_{t+1}(\thetastar_{t+1})\)(s,a)| + \sup_{(s,a)} | \( \T_t Q_{t+1}(\thetastar_{t+1})\)(s,a)  - \phi_t(s,a)^\top\thetastar_t| 
	\numberthis{\label{eqn:52}}
\end{align*}
Since $\thetastar_{t+1} \in \mathcal B_{t+1}$ by construction (see \cref{def:BestApproximant}), $Q_{t+1}(\thetastar_{t+1})\in \mathcal Q_{t+1}$ and so by definition of inherent Bellman error (and \cref{def:BestApproximant}) the second term must be $\leq \IBE$. It remains to examine the first term. By definition of Bellman operator $\T_t$ we have that for any $(s,a)$ pair:
\begin{align}
|\( \T_t\Qstar_{t+1}\)(s,a) - \( \T_t Q_{t+1}(\thetastar_{t+1})\)(s,a)| & =  |r_t(s,a) + \E_{s'\sim p_t(s,a)}\max_{a'} \Qstar_{t+1}(s',a')  - r_t(s,a) - \E_{s'\sim p_t(s,a)}\max_{a'}\phi_{t+1}(s',a')^\top\thetastar_{t+1} | \\
&  \leq  | \E_{s'\sim p_t(s,a)}\max_{a'}  \phi_{t+1}(s',a')^\top\thetastar_{t+1} - \max_{a'} \Qstar_{t+1}(s',a')  | \\
&  \leq \E_{s'\sim p_t(s,a)}|\max_{a'} \phi_{t+1}(s',a')^\top\thetastar_{t+1} - \max_{a'} \Qstar_{t+1}(s',a')  | \\
& \leq \E_{s'\sim p_t(s,a)}\max_{a'}| \phi_{t+1}(s',a')^\top\thetastar_{t+1} -  \Qstar_{t+1}(s',a')  | 
\leq  (H-t)\IBE.
\end{align}
The last inequality in the previous display comes from the inductive hypothesis, and concludes the proof.
\end{proof}

\newpage
\subsection{Optimism}

The purpose of this section is to show that if  assumption \ref{ass:MainAssumption} is satisfied, then the program of \cref{def:PlanningOptimizationProgram} 1) admits a feasible solution and 2) the solution returned is at least as good as the $\thetastar$'s defined in \cref{def:BestApproximant}, which is in some sense the best possible.
\begin{lemma}[Optimism]
\label{lem:Optimism}
Outside of the failure event $F_k$, $\(\thetastar_1,\dots,\thetastar_H\)$
is a feasible solution\footnote{The solution comprises also the $\widehat \theta$ and $\overline \xi$ variables, so this is ``part of'' a feasible solution} to the program of \fullref{def:PlanningOptimizationProgram} in episode $k$. As a consequence the value function returned by the algorithm $\overline V_1(s_{1k})$ satisfies 
\begin{align}
	\overline V_1(s_{1k}) \geq \Vstar_1(s_{1k}) - H\IBE.
\end{align}

\end{lemma}

\begin{proof}
First we show feasibility, and then the estimation bound.
\paragraph{Feasibility}
The proof is constructive: we show that we can find $\overline \xi_1,\dots,\overline \xi_H$ so that we can satisfy $\overline \theta_t = \thetastar_t$ for all $t\in[H]$ along with the other constraints of the program of \cref{def:PlanningOptimizationProgram}.  The base case $t=H+1$ is trivial, as $\overline \theta_{H+1} = \thetastar_{H+1} = 0$ already holds. The inductive hypothesis goes backward from $t = H$ to $t=1$ and consists of the following statement:

\emph{There exists $\overline \xi_{t},\dots,\overline \xi_{H}$ such that:}
\begin{itemize}
	\item \emph{$\overline \theta_{t} = \thetastar_t,\dots,\overline \theta_{H} = \thetastar_H$}
	\item \emph{the constraints of the program of \cref{def:PlanningOptimizationProgram} are satisfied for $t,\dots,H$}
	\item \emph{no additional constraints are set on $\overline \theta_{\tau}, \widehat \theta_{\tau},\overline \xi_{\tau}$ for $\tau = 1,\dots,t-1$.}
\end{itemize}
 
Now assume the inductive hypothesis holds at $t+1$. We have from \cref{lem:FirstStepAnalysis} the  relation below. Here we set $\overline \theta_{t+1} = \thetastar_{t+1}$ using the inductive hypothesis, and we \emph{request} $\overline \theta_{t}= \thetastar_t$ to show the inductive step:
\begin{align}
\thetastar_t =  \overline \xi_{t} +  \underbrace{\mathring \theta_t(\T_t Q_{t+1}(\thetastar_{t+1}))}_{\defeq \thetastar_t} + \Sigma^{-1}_{tk}\sum_{i=1}^{k-1} \phi_{ti}^\top 
	\mathring \Delta_{t}(Q_{t+1}(\thetastar_{t+1}))(s_{ti},a_{ti}) - \lambda\Sigma^{-1}_{tk}\mathring \theta_t(\T_t Q_{t+1}(\thetastar_{t+1}))  + \Sigma^{-1}_{tk}\sum_{i=1}^{k-1} \phi_{ti} \eta_{ti}(V_{t+1}(\thetastar_{t+1})) 
\end{align}
Notice that $\thetastar_{t} \in \mathcal B_t$ by definition of $\thetastar_t$ and simplifying the above display gets us the following condition to satisfy for $\overline \xi_t$:
\begin{align}
\label{eqn:xi_val}
\overline \xi_{t}  = - \Sigma^{-1}_{tk}\sum_{i=1}^{k-1} \phi_{ti}^\top 
	\mathring \Delta_{t}(Q_{t+1}(\thetastar_{t+1}))(s_{ti},a_{ti}) + \lambda\Sigma^{-1}_{tk}\mathring \theta_t(\T_t Q_{t+1}(\thetastar_{t+1}))  - \Sigma^{-1}_{tk}\sum_{i=1}^{k-1} \phi_{ti} \eta_{ti}(V_{t+1}(\thetastar_{t+1})). 
\end{align}
Taking $\Sigma_{tk}$-norms\footnote{In particular, note that $\Sigma_{tk}$ is spd} and using the triangle inequality we get:
\begin{align}
\label{eqn:xi_ActualBound}
 \|\overline \xi_{t}\|_{\Sigma_{tk}} & \leq \| \sum_{i=1}^{k-1} \phi_{ti}^\top 
	\mathring \Delta_{t}(Q_{t+1}(\thetastar_{t+1}))(s_{ti},a_{ti})  \|_{\Sigma^{-1}_{tk}} + \lambda\|\mathring \theta_t(\T_t Q_{t+1}(\thetastar_{t+1}))  \|_{\Sigma^{-1}_{tk}} + \| \sum_{i=1}^{k-1} \phi_{ti} \eta_{ti}(V_{t+1}(\thetastar_{t+1})) \|_{\Sigma^{-1}_{tk}}.
	\end{align}
Since $\thetastar_{t+1} \in \mathcal B_{t+1}$ by definition, we know that $V_{t+1}(\thetastar_{t+1}) \in \mathcal V_{t+1}$ and therefore outside of the failure event of \cref{def:FailureEvent} we know that:
\begin{align}
	 \| \sum_{i=1}^{k-1} \phi_{ti} \eta_{ti}(V_{t+1}(\thetastar_{t+1})) \|_{\Sigma^{-1}_{tk}} & \leq \sqrt{\beta_{tk}}.
\end{align}
It remains to bound the other two terms in the rhs of \cref{eqn:xi_ActualBound}.
An application of \cref{lem:WorstCaseBound} gives one of the two bounds:
\begin{align}
\lambda\|\mathring \theta_t(\T_t Q_{t+1}(\thetastar_{t+1}))  \|_{\Sigma^{-1}_{tk}} \leq \sqrt{\lambda}\|\mathring \theta_t(\T_t Q_{t+1}(\thetastar_{t+1})) \|_2 \leq \sqrt{\lambda}\Radius_t.
\end{align}
The last equality holds by definition of the operator $\mathring \theta_t\( \cdot \)$. Next \cref{lem:ProjectionBound} helps bound the remaining term:
\begin{align}
\| \sum_{i=1}^{k-1} \phi_{ti}^\top 
	\mathring \Delta_{t}(Q_{t+1}(\thetastar_{t+1}))(s_{ti},a_{ti})  \|_{\Sigma^{-1}_{tk}} \leq \sqrt{k}\IBE.
	\end{align}
Combining the above relations and plugging back into \cref{eqn:xi_ActualBound} gives us that to satisfy \cref{eqn:xi_val}, the $\Sigma_{tk}$-norm of $\overline \xi_{t}$ must satisfy:
\begin{align}
	& \|\overline \xi_{t} \|_{\Sigma_{tk}} \leq \sqrt{\beta_{tk}} + \sqrt{k}\IBE + \sqrt{\lambda}\Radius_t \defeq \sqrt{\alpha_{tk}}
\end{align}
This is the definition of $\alpha_{tk}$. Since $\overline \theta_t = \thetastar_t \in \mathcal B_t$ holds, we have shown we can satisfy all constraints of the program of \cref{def:PlanningOptimizationProgram} at timestep $t$ by fixing the value of $\overline \xi_t$, without adding further constraints to the optimization variables for $\tau < t$.

We have shown that the inductive hypothesis holds $\forall t\in[H]$, so in particular for $t=1$. The suboptimality gap result follows from the fact that the optimization program finds a solution with a value at least as high as $\max_a \phi_{t}(s_{1k},a)^\top\thetastar_1$ for the starting state $s_{1k}$, as explained next.

\paragraph{Estimation Bound}
Denote with $\{\overline \theta_{tk}\}_{t=1,\dots,H}$ the maximizer found in episode $k$, and with $\Vbar_{tk},\Qbar_{tk}$ the corresponding value  and action-value function, respectively.
Since $\thetastar_1$ is a feasible solution,
\begin{align}
\Vbar_{1k}(s_{1k}) & = \max_{a'}\Qbar_{1k}(s_{1k},a') \\
& = \max_{a'} \phi_1(s_{1k},a')^\top\overline \theta_{1k} \\
& \geq \max_{a'} \phi_1(s_{1k},a')^\top\thetastar_1
\end{align}
otherwise $\overline \theta_{1k}$ would not be a maximizer,
\begin{align}
& \geq  \phi_1(s_{1k},\pistar_1(s_{1k}))^\top\thetastar_1 \\
 & \geq \Qstar_1(s_{1k},\pistar_1(s_{1k})) -H\IBE \\
 & = \Vstar_1(s_{1k}) -H\IBE
\end{align}
where the last inequality is by \cref{lem:AccuracyBoundThetastar}.
\end{proof}

\newpage
\subsection{Regret Bound}
We are finally ready to present our regret bound:

\MainResult
\begin{proof}
First, decompose the regret as 
\begin{align}
\textsc{Regret}(T) \defeq  \sum_{k=1}^K \(\Vstar_1 - V_1^{\pi_k}\)(s_{1k}) = \sum_{k=1}^K \(\Vstar_1 - V_1^{\pi_k}\)(s_{1k})\1(\overline F_k) + \sum_{k=1}^K \(\Vstar_1 - V_1^{\pi_k}\)(s_{1k})\1(F_k).
\end{align}
The second sum in the rhs above is non-zero only when at least one indicator $\1(F_k)$ turns on for at least one $k$. This event can be written as $\bigcup_{k\in[K]} F_k$, and following \cref{lem:FailureEventProbability} we can bound its size:
\begin{align}
\Pro\(\exists k\in[K] \; \textrm{s.t.} \; F_k\)  = \Pro\(\bigcup_{k\in[K]} F_k\) \leq \frac{\delta}{2}.
\end{align}
Thus it's sufficient to bound the regret when $\bigcup_{k\in[K]} F_k$ does not occur and consider:
\begin{align}
\label{eqn:OptiEqn01}
\sum_{k=1}^K \(\Vstar_1 - V_1^{\pi_k}\)(s_{1k})\1(\overline F_k).
\end{align}
We indicate with $\pi_k$ the policy found by \fullref{alg:AlgoLabel} in episode $k$. Thanks to \fullref{lem:Optimism} we can ensure this is nearly-optimistic:
\begin{align}
\label{eqn:OptiEqn02}
    \(\Vstar_1 - V_1^{\pi_k}\)(s_{1k})\1(\overline F_k) = \underbrace{\(\Vstar_1 - \overline V_{1k} \)(s_{1k})\1(\overline F_k)}_{\leq H\IBE}  + \(\overline V_{1k} - V_1^{\pi_k}\)(s_{1k})\1(\overline F_k).
\end{align}
We put the expression above aside for a second to derive a recursion. First notice the equality below:
\begin{align}
\label{eqn:eqn022}
(\T_t\Qbar_{t+1,k})(s_{tk},a_{tk}) - V^{\pi_k}_{t}(s_{tk}) = \E_{s' \sim p_t(s_{tk},a_{tk})}\(\Vbar_{t+1,k}-V^{\pi_k}_{t+1}\) (s').
\end{align}
Now evaluate \fullref{lem:FirstStepAnalysis} (with $s = s_{tk}$ and $a = a_{tk} = \pi_{tk}(s_{tk})$ for short) under $\overline F_k$:
\begin{align}
\overline Q_{tk} (s_{tk},a_{tk}) \leq  \T_t\Qbar_{t+1,k}(s_{tk},a_{tk}) + \IBE + \|\phi_t(s_{tk},a_{tk})\|_{\Sigma^{-1}_{tk}}\( \sqrt{k}\IBE + \sqrt{\alpha_{tk}} + \sqrt{\beta_{tk}} + \sqrt{\lambda} \Radius_t\).
\end{align}
Recalling that $\overline Q_{tk} (s_{tk},a_{tk}) = \Vbar_{tk} (s_{tk})$ and combining the two above displays to eliminate $ \T_t\Qbar_{t+1,k}(s_{tk},a_{tk})$ gives 
\begin{align}
\label{eqn:77}
\(\Vbar_{tk} - V_{t}^{\pi_k}\)(s_{tk}) \leq & \E_{s' \sim p_t(s_{tk},a_{tk})}\(\Vbar_{t+1,k}-V^{\pi_k}_{t+1}\) (s') + \IBE + \| \phi_t(s_{tk},a_{tk}) \|_{\Sigma^{-1}_{tk}} \(\sqrt{k}\IBE + \sqrt{\alpha_{tk}} + \sqrt{\beta_{tk}} + \sqrt{\lambda} \Radius_t \).
\end{align}
We can define the martingale:
\begin{align}
    \dot \zeta_{tk} \defeq \dotzetadef.
\end{align}
Next, we plug the martingale definition into \cref{eqn:77}, use induction over $t$, and finally substitute back in \fullref{eqn:OptiEqn02}. Further summation over the episodes $k$ gives:
\begin{align}
& \sum_{k=1}^K  \(\Vstar_1 - V_1^{\pi_k}\)(s_{1k})\1\(\overline F_k\) \leq  HK\IBE + \\
& + \sum_{k=1}^K \sum_{t=1}^H \Bigg[ \dot \zeta_{tk} + \IBE + \| \phi_{tk} \|_{\Sigma^{-1}_{tk}} \(\sqrt{k}\IBE + \sqrt{\alpha_{tk}} + \sqrt{\beta_{tk}} + \sqrt{\lambda} \Radius_t\) \Bigg]\1\(\overline F_k\)
\end{align}
Further applying Cauchy-Schwartz to the term featuring $\| \phi_{tk} \|_{\Sigma^{-1}_{tk}}$ gives:
\begin{align}
& \leq 2\IBE T + \sum_{k=1}^K \sum_{t=1}^H \dot \zeta_{tk}\1\(\overline F_k\) +\sum_{t=1}^H \sqrt{K}\sqrt{ \sum_{k=1}^K  \| \phi_t(s,a) \|^2_{\Sigma^{-1}_{tk}} \(\sqrt{K}\IBE + \sqrt{\alpha_{tk}} + \sqrt{\beta_{tk}} + \sqrt{\lambda} \Radius_t\)^2 } 
\end{align}
We can right away substitute $\beta_{tk}\leq \beta_{tK} = \widetilde O(\sqrt{d_t + d_{t+1}})$ and $\alpha_{tk} \leq \alpha_{tK}$.
Since $\overline V_{t+1,k}(s) = \phi_t(s,a)^\top\theta_t$ for some action $a$ and $\|\theta_t\|_2\leq\Radius_t$ we have that $\overline V_{t+1,k}(s) \leq \Lphi\Radius_{t+1} \leq \sqrt{d_{t+1}}$ by Cauchy-Schwartz and assumption \ref{ass:MainAssumption}. Azuma-Hoeffding (\cref{prop:Azuma}) with a union bound over $\kappa\in[K]$ ensures (notice that by assumption \ref{ass:MainAssumption} we also have that $\|V_{t+1}^{\pi_k}\|_\infty \leq 1$):
\begin{align}
\label{eqn:AzumaStatement}
\Pro\( \exists \kappa \in [K]  \quad  \textrm{such that} \quad \Big| \sum_{k=1}^{\kappa} \dot \zeta_{tk} \Big| > \sqrt{2\( 2\Lphi\Radius_{t+1}\)^2\kappa\ln\( \frac{2T}{\delta} \)} \) \leq \frac{\delta}{2}.
\end{align}
Thus, with high probability the martingale gives a contribution $\widetilde O(\sum_{t=1}^H  \sqrt{d_{t+1} K}) = \widetilde O(\sum_{t=1}^H  \sqrt{d_{t} K})$ since $d_{H+1} = 0$.

Finally, lemma 11 in the appendix of \cite{Abbasi11} gives with $\lambda = 1$ and $\Lphi = 1$:
\begin{align}
	 \sum_{k=1}^K \| \phi_{tk} \|^2_{\Sigma^{-1}_{tk}} & \leq   2\(d_t\ln\(\(\underbrace{\textrm{trace}(\lambda I)}_{= d_t } + K \Lphi^2\)/d_t\) - \underbrace{\ln\det(\lambda I)}_{= 0} \) = \widetilde O(d_t).
	 \end{align}
	This concludes the regret bound, which holds with probability at least $1-\delta$ jointly over all episodes by union-bounding the failure event in \cref{lem:FailureEventProbability} with \cref{eqn:AzumaStatement}, and substituting $\Radius_t \leq \sqrt{d_t}$, $\Lphi = 1$, $\lambda = 1$. By using $\sqrt{d_t + d_{t+1}} \leq \sqrt{d_t} + \sqrt{d_{t+1}}$ and $\sqrt{d_td_{t+1}} \leq \sqrt{d^2_{t} + d^2_{t+1}} \leq d_t + d_{t+1}$ and that $d_{H+1} = 0$ we obtain:
	\begin{align}
	\textsc{Regret(K)} & \leq \widetilde O \(T\IBE + \sum_{t=1}^H \sqrt{d_t K} + \sum_{t=1}^H\sqrt{d_t}\sqrt{K} \(\sqrt{K}\IBE + \sqrt{d_t+d_{t+1}}  + \sqrt{d_t}\) \) \\
	& = \widetilde O \(T\IBE + \sum_{t=1}^H \sqrt{d_t K} + \sum_{t=1}^H\sqrt{d_t} \IBE K   + \sum_{t=1}^H\sqrt{d_t} \sqrt{K} \sqrt{d_t+d_{t+1}} \) \\
	& = \widetilde O \( \sum_{t=1}^H \sqrt{d_t}\IBE K + \sum_{t=1}^H \sqrt{K}\sqrt{d_t}(\sqrt{d_t}+\sqrt{d_{t+1}})  \) \\
	& = \widetilde O \( \sum_{t=1}^H \sqrt{d_t}\IBE K + \sum_{t=1}^H \sqrt{K}d_t + \sum_{t=1}^H \sqrt{K}\sqrt{d_t d_{t+1}})  \)\\
	& = \widetilde O \( \sum_{t=1}^H \sqrt{d_t}\IBE K + \sum_{t=1}^H \sqrt{K}d_t + \sum_{t=1}^H \sqrt{K} \sqrt{d_t^2 + d_{t+1}^2 }  \) \\
		& = \widetilde O \( \sum_{t=1}^H \sqrt{d_t}\IBE K + \sum_{t=1}^H \sqrt{K}d_t + \sum_{t=1}^H \sqrt{K} \(d_t + d_{t+1}\) \) \\
		& = \widetilde O \( \sum_{t=1}^H \sqrt{d_t}\IBE K + \sum_{t=1}^H d_t\sqrt{K}\).
\end{align} 
\end{proof}

\newpage
\subsection{Projection Bound}
The purpose of this section is to compute the maximum amplification factor of the model misspecification while using a least-square procedure. While in the generative model setting this has been analyzed before \cite{zanette19limiting,lattimore2020learning} with an amplification-factor that can be made at most as large as $\sqrt{d}$ by using the Kiefer–Wolfowitz theorem \cite{lattimore2020bandit}. Unfortunately in the online setting one cannot choose the features and the the amplification factor can grow with $\sqrt{n}$ where  $n$ is the number of samples. However, one can show that this situation cannot persist for long in the online setting. Below we analyze one technical factor in the prediction error. We use a geometric argument based on a shrinking projector.

\begin{lemma}[Projection Bound]
\label{lem:ProjectionBound}
Let $\{a_i\}_{i=1,\dots,n}$ be any sequence of vectors in $\R^d$ and $\{b_i\}_{i=1,\dots,n}$ be any sequence of scalars such that $|b_i| \leq \epsilon \in \R^+$. 
For any $\lambda \geq 0$ and $k \in \mathbb N$ we have:
\begin{align}
\Bigg\|\sum_{i=1}^n a_i b_i \Bigg\|^2_{\big[\sum_{i=1}^{n}a_i a_i^\top + \lambda I\big]^{-1}}
\leq n\epsilon^2.	
\end{align}
\end{lemma}
Notice that in this proof $\Sigma$ is the matrix of singular values defined according to standard linear algebra notation and is not the covariance matrix used elsewhere in this work.
\begin{proof}
Consider the matrix $A\in\R^{n\times d}$ such that $A[i,:] = a_i^\top$, and the vector $b\in\R^n$ with $b[i] = b_i$ and consider the full SVD $A = U\Sigma V^\top$, with $U \in \R^{n\times n}$, $\Sigma \in \R^{n\times d}$, $V \in \R^{d\times d}$. Here $U$ and $V$ are orthogonal matrices and also define $s$ to be the number of non-zero singular values, so that $s \leq \min\{n,d\}$. For an existence proof of such decomposition see for example Thm 2.4.1 in \cite{golub2012matrix}. By definition, the singular values in $\Sigma$ are decreasing in value, so we can write:
\begin{align}
	U\Sigma V^\top = 
    	\begin{bmatrix}
		U_1 & U_{2}
	\end{bmatrix}
	\begin{bmatrix}
		\Sigma_{11} & \Sigma_{12} \\
		\Sigma_{21} & \Sigma_{22}
	\end{bmatrix}
	  \begin{bmatrix}
		V_1^\top \\
		V_2^\top
	\end{bmatrix}
	= U_1 \Sigma_{11}V_1^\top
\end{align}
with $\Sigma_{11} \in \R^{s\times s}$, $0 = \Sigma_{12} \in \R^{s\times (d-s)}$, $0 = \Sigma_{21} \in \R^{(n-s)\times s}$, $0 = \Sigma_{22} \in \R^{(n-s)\times (d-s)}$.
The reader can verify that $ A^\top b = \sum_{i=1}^{n} a_i b_i $ and $A^\top A = \sum_{i=1}^{n}a_i a_i^\top$. Using this, and the definition of $\big[A^\top A + \lambda I \big]^{-1}$-norm we can write:
\begin{align}
\|\sum_{i=1}^n a_i b_i \|^2_{\big[\sum_{i=1}^{n}a_i a_i^\top + \lambda I\big]^{-1}} & = \|A^\top b \|^2_{\big[A^\top A + \lambda I \big]^{-1}} = b^\top A\big[A^\top A + \lambda I\big]^{-1}A^\top b.
\end{align}
Now it's time to use the SVD of $A$  while recalling $VV^\top = V^\top V = I$ and $U^\top U = I$, yielding:
\begin{align*}
	& b^\top \underbrace{U\Sigma V^\top}_{A}\big[\underbrace{V\Sigma^\top U^\top}_{A^\top} \underbrace{U\Sigma V^\top}_{A} + \lambda \underbrace{VV^\top}_{I} \big]^{-1} \underbrace{V\Sigma^\top U^\top}_{A^\top} b \\
	& b^\top U\Sigma V^\top\big[V\Sigma^\top \Sigma V^\top + \lambda VV^\top \big]^{-1} V\Sigma^\top U^\top b \\
	& b^\top U\Sigma V^\top V \big[\Sigma^\top \Sigma + \lambda I  \big]^{-1} V^\top V\Sigma^\top U^\top b  \\
		& b^\top U\Sigma  \big[\Sigma^\top \Sigma + \lambda I  \big]^{-1} \Sigma^\top U^\top b. 
		\numberthis{\label{eqn:eqn102}}
\end{align*}
Since we can write:
\begin{align}
\Sigma^\top U^\top b = 
\begin{bmatrix}
\Sigma^\top_{11} & \Sigma^\top_{12} \\
\Sigma^\top_{21} & \Sigma^\top_{22}
\end{bmatrix}
\begin{bmatrix}
U_1^\top b \\
U_2^\top b 
\end{bmatrix} 
=
\begin{bmatrix}
\Sigma_{11} & 0 \\
0 & 0
\end{bmatrix}
\begin{bmatrix}
U_1^\top b \\
U_2^\top b 
\end{bmatrix} = 
\begin{bmatrix}
\Sigma_{11} U_1^\top b \\
0
\end{bmatrix}
\end{align}
from \cref{eqn:eqn102} we can write:
\begin{align*}
& = \begin{bmatrix}
 b^\top U_1 \Sigma_{11}^\top & 0
\end{bmatrix}
	\begin{bmatrix}
\Sigma_{11}^\top \Sigma_{11} + \lambda I & 0 \\
0 & \lambda I
\end{bmatrix}^{-1}
\begin{bmatrix}
\Sigma_{11} U_1^\top b \\
0
\end{bmatrix} 
\\ 
& =
\begin{bmatrix}
 b^\top U_1 \Sigma_{11}^\top & 0
\end{bmatrix}
	\begin{bmatrix}
\( \Sigma_{11}^\top \Sigma_{11} + \lambda I\)^{-1} & 0 \\
0 & \( \lambda I \)^{-1}
\end{bmatrix}
\begin{bmatrix}
\Sigma_{11} U_1^\top b \\
0
\end{bmatrix} \\
& =
\begin{bmatrix}
 b^\top U_1 \Sigma_{11}^\top & 0
\end{bmatrix}
	\begin{bmatrix}
\( \Sigma_{11}^\top \Sigma_{11} + \lambda I\)^{-1} \Sigma_{11} U_1^\top b \\
0 
\end{bmatrix} \\
& =  \underbrace{b^\top U_1}_{\defeq x^\top} \Sigma_{11}^\top \( \Sigma_{11}^\top \Sigma_{11} + \lambda I\)^{-1} \Sigma_{11} \underbrace{U_1^\top b}_{\defeq x}. 
\numberthis{\label{eqn:111}}
\end{align*} 
Notice that, by construction, $\Sigma_{11}$ an $s \times s$ is a diagonal matrix filled of non-zeros. 
\begin{align}
\label{eqn:eqn008}
\Sigma_{11}^\top \( \Sigma_{11}^\top \Sigma_{11} + \lambda I\)^{-1} \Sigma_{11}	  = \Sigma_{11} \( \Sigma_{11}^2 + \lambda I\)^{-1} \Sigma_{11}.
\end{align}
Indicate with $d_i$  the $i$-th diagonal element of the matrix in \cref{eqn:eqn008} which reads:
\begin{align}
\Sigma_{11}[i,i] \( \Sigma_{11}[i,i]^2 + \lambda I\)^{-1} \Sigma_{11}[i,i] \defeq d_i \leq 1.
\end{align}
The inequality is because $\Sigma_{11}[i,i] > 0$ by construction and $\lambda > 0$. In essence, we have obtained from \cref{eqn:111} the $d$-weighted $2$-norm of $x$: 
\begin{align}
	= \sum_{i=1}^{s} d_i \(x[i]\)^2 & \leq \sum_{i=1}^{s} \(x[i]\) ^2 \\
	& = \| x \|^2_2 \\
	& = \|U_1^\top b \|_2^2 \\
	& = \Bigg\| 	
	\begin{bmatrix}
		U_1^\top b \\
		0
	\end{bmatrix}  \Bigg\|^2_2 \\
	& \leq \Bigg\| 	
	\begin{bmatrix}
		U_1^\top b \\
		U_2^\top b
	\end{bmatrix}  \Bigg\|^2_2 \\
	& = \| U^\top b \|_2^2 \\
	& = b^\top U U^\top b  \\
	& = b^\top b  \\
	& = \| b \|_2^2 \\
	& = \sum_{i=1}^{n} (b[i])^2 \leq \sum_{i=1}^{n} \epsilon^2 = n \epsilon^2.
\end{align}
\end{proof}

\subsection{Technical Lemmas}
\begin{lemma}[Worst-Case Bound]
\label{lem:WorstCaseBound}
For any vector $x \in R^d$ it holds that:
\begin{align}
    \|x\|_{\Sigma^{-1}_{tk}} \leq \frac{1}{\sqrt{\lambda}}\|x\|_2.
\end{align}
\end{lemma}
\begin{proof}
Unless $x = 0$, in which case the statement holds, we can write:
\begin{align}
\frac{\|x\|_{\Sigma^{-1}_{tk}}}{\| x \|_2} & = \sqrt{\frac{x^\top{\Sigma^{-1}_{tk}} x}{x^\top x}} \leq \sqrt{\lambda_{max} (\Sigma^{-1}_{tk})} = \frac{1}{\sqrt{\lambda_{min}(\Sigma_{tk})}} = \frac{1}{\sqrt{\lambda}}
\end{align}
The inequality is due to, for example, the Courant-Fischer minimax theorem (see Theorem 8.1.2 in \cite{golub2012matrix}), and $\lambda_{max},\lambda_{min}$ are the maximum and minimum eigenvalues of the matrix in parenthesis, respectively.
\end{proof}

\newpage
\section{Lower Bounds}
\label{sec:LowerBounds}
In this section we first recall the classical linear bandit ``statistical'' lower bound (in the absence of misspecification) and the recent lower bound by \cite{du2019good} regarding misspecified linear bandits. Then we embed these into an MDP to provide a reinforcement learning lower bound for our setting. At a high level the construction works at follows: the starting states is chosen from two sets of non-communicating states: in set $L$ (for linear) the agent encounters a linear bandit problem (which can be represented within our framework), that induces a $\Omega(\sum_{t=1}^H d_t \sqrt{K})$ regret; in set $M$ we use a sequence of misspecified linear bandit problems, each with misspecification $\epsilon$ (which is also the inherent Bellman error $\IBE$), and this gives an expected regret at least of order $\Omega(\sum_{t=1}^H \sqrt{d_t}\IBE K)$ for any algorithm. Since the agent is forced to go through either set of problems a lower bound $\Omega(\sum_{t=1}^H d_t \sqrt{K} + \sum_{t=1}^H \sqrt{d_t}\IBE K)$ follows.

\subsection{Statistical Lower Bound}
In this section we mention the construction that supports the lower bound of \cref{prop:LowerBoundNoMisspec}. Since our MDP framework includes bandit problems, it is sufficient to consider a linear bandit problem to achieve the result. We recall the following result (theorem 24.2 in \cite{lattimore2020bandit}) with our notation:
\begin{lemma}[Stochastic Linear Bandit Unit Ball Lower Bound]
\label{prop:BanditLowerBound}
Consider the class of linear bandit problems with reward function $\phi^\top\thetastar + \eta$ where $\eta$ is $1$ (conditionally) sub-Gaussian noise.
Assume $\frac{d^2}{48} \leq K$ where $K$ is the time elapsed and let the feature set be $\{\phi \in \R^{d} \mid \| \phi \|_2 \leq 1  \}$. Then for any algorithm there exists a parameter vector $\thetastar \in \R^{d}$ with $\|\thetastar\|^2_2=\frac{d^2}{48K} \leq 1$ such that:
\begin{align}
	K\max_{\phi}\phi^\top\thetastar - \E\Bigg[\sum_{t=1}^{K} \phi_t^\top \thetastar \Bigg] \geq d\sqrt{K}/(16\sqrt{3})
\end{align}
where $\phi_1,\dots,\phi_K$ are the features selected by the algorithm.
\end{lemma}

The result of \cref{prop:LowerBoundNoMisspec} is a direct consequence of \cref{prop:BanditLowerBound}. In particular, consider an MDP with a linear bandit reward response with features in the unit ball at the initial state $s_{start}$ and deterministic transitions to a terminal state $s_{end}$ where only one action $a_{end} $ exists. For $t > 1$ we choose $\phi_t(s_{end},a_{end}) = 1$ (so $d_2=\dots = d_H = 1)$; no reward is present in $s_{end}$ and the transition is to $s_{end}$. This problem has dimensionality $\widetilde d = d + \sum_{t=2}^H 1 = d + H-1$, and satisfies assumption \ref{ass:MainAssumption}. The statement of the theorem follows immediately. 

\subsection{Misspecification Lower Bound}
In this section we recall the bandit lower bound recently proposed by \cite{du2019good}. We follow the presentation in the technical note by \cite{lattimore2020learning} for simplicity of presentation. We use a rescaling argument to ensure the actual rewards are in $[0,a]$ (with $a \approx \frac{1}{H}$) so that we can later stack $H$ of them while still complying with assumption \ref{ass:MainAssumption} regarding the maximum return. 

Assuming (finitely many) $A$ actions, the reward of playing action $a$ at timestep $t$ in the only possible state is synthetically summarized as the $\mu_{a}$ entry in $\mu \in \R^A$. Let the hypothesis class $\mathcal H$ be the set of all possible reward responses $\mathcal H \defeq \{\mu \in \R^A \mid \mu \in [0,a]^A \}$. We define the worst-case expected query complexity for any algorithm $\mathscr A $ to output a $\delta$-correct action (an action $i$ such that $\max_{j} \mu_{j} - \mu_{i} \leq \delta $):
\begin{align}
	c_\delta(\mathcal H) \defeq \inf_{\mathscr A} \sup_{\mu \in \mathcal H} q_\delta(\mathscr A,\mu).
\end{align}
where $q_\delta(\mathscr A,\mu)$ is the expected query complexity for $\mathscr A$ to return at least a $\delta$-suboptimal action on the problem instance identified by $\mu$.
The following can be derived by elementary probability using symmetry, where $e_i$ is the $i$-th canonical vector.
\begin{lemma}[Lemma 2.1 in \cite{lattimore2020learning}]
\label{lem:cdelta}
For any $a > 0$,
\begin{align}
	c_\delta(\{a e_1,\dots,a e_A \}) \geq \frac{A+1}{2}, \quad \forall \delta \in [0,a].
\end{align}
\end{lemma}
Next, notice that bigger hypothesis classes can only increase the sample complexity:
\begin{lemma}
\label{lem:BiggerIsBigger}
If $U \subset V$ then $c_{\delta}(U) \leq c_\delta (V)$.
\end{lemma}

We have the following consequence of the Johnson-Lindenstrauss lemma (here $\epsilon'$ is a just an intermediate quantity we define, it is not the accuracy $\epsilon$ of the predictor as in \cite{lattimore2020learning}; we define such accuracy later):
\begin{lemma}[Lemma 3.1 from \cite{lattimore2020learning}]
\label{lem:JL}
For any $\epsilon' > 0$ and $d \in [A]$ such that $d \geq \ceil{\frac{8\ln(A)}{(\epsilon')^2}}$ there exists $\Phi \in \R^{A \times d}$ with unique rows such that 
(here $\Phi[i,:]$ indicates the $i$-th row of $\Phi$) for all $i \neq j$:
\begin{align}
\label{eqn:JLproperties}
	\| \Phi[i,:] \|_2 =1  \quad \text{and} \quad |\Phi[i,:]^\top \Phi[j,:]| \leq \epsilon'.
\end{align}
\end{lemma}

We define the hypothesis class defined by $\Phi$ and perturbed in the hypercube $[-\epsilon,+\epsilon]^A$:
\begin{align}
\label{eqn:H}
\mathcal H_{\Phi,a}^{\epsilon} \defeq \{ (\Phi\theta + c) \in \R^A \mid \; \theta \in \R^d, \| \theta \|_2 \leq a, \; c \in [-\epsilon,\epsilon]^A\}.
\end{align}

Combining \cref{lem:cdelta,lem:BiggerIsBigger,lem:JL} gives (here $\epsilon$ is the ``\emph{approximation error}''):
\begin{lemma}[Slight generalization of proposition 3.2 in \cite{lattimore2020learning}]
\label{prop:rescaled_lowerbound}
For any $\epsilon>0$ and $d\in[A]$
there exists $\Phi \in \R^{A \times d}$ with rows of unitary $2$-norm such that $c_{\delta}(\mathcal H^{\epsilon}_{\Phi,a}) \geq \frac{A+1}{2}$ for any $\delta \in [0,a]$ with $a = \epsilon \sqrt{\frac{d-1}{8\ln(A)}}$. 
\end{lemma}

\begin{proof}
Fix $\epsilon' = \sqrt{\frac{8\ln A}{d-1}}$ and let $\Phi \in \R^{A \times d}$ be the matrix given in \cref{lem:JL} (as function of $\epsilon'$). Denote $\theta = a\Phi[i,:]$ for a positive $a \in \R$. Lemma \ref{lem:JL} (in particular, \cref{eqn:JLproperties}) ensures
\begin{align*}
|\Phi[i,:]^\top \theta | = a |\Phi[i,:]^\top \Phi[i,:] | & = a \\
|\Phi[j,:]^\top \theta | = a |\Phi[j,:]^\top \Phi[i,:] | & \leq a\epsilon' \quad j \neq i.	
\numberthis{\label{eqn:bothofthem}}
\end{align*}
Therefore, fix any index $i \in [A]$, which identifies a canonical vector $e_i \in \R^A$, i.e., a vector with a $1$ in position $i$ and $0$ elsewhere. We have that $\theta = a\Phi[i,:]$ satisfies $\theta \in \R^d, \|\theta \|_2 = a$. In addition there exists $c \in [-a\epsilon',a\epsilon']^A$ such that $\Phi\theta + c = a e_i$ (by leveraging \cref{eqn:bothofthem}). Therefore $a e_i \in \mathcal H_{\Phi,a}^{a \epsilon'}$. In other words, there exists a matrix $\Phi$, function of $\epsilon'$ and an appropriate $\theta$, which depends on $i$, such that $\Phi \theta$ can approximately represent the (scaled) canonical vector $ae_i$ up to an additive error of order $a\epsilon' \defeq \epsilon$.
As explained, we can set $\epsilon' = \sqrt{\frac{8\ln A}{d-1}}$ to obtain this; therefore $\epsilon = a \epsilon' =  a \sqrt{\frac{8\ln A}{d-1}}$ yields $a = \epsilon \sqrt{\frac{d-1}{8\ln(A)}}$. Since we have reasoned for an arbitrary $i$, we have that $\{e_1,\dots,e_A \} \subset  \mathcal H_{\Phi,a}^{\epsilon}$. At this point invoke \cref{lem:cdelta,lem:BiggerIsBigger} to obtain $c_{\delta}(\mathcal H^{\epsilon}_{\Phi,a}) \geq \frac{A+1}{2}$ for $\delta \in [0,a]$.
\end{proof}

\paragraph{Remark on regret}
By the symmetry of the problem, a fraction of $(1-\frac{1}{A})$ queries in expectation must be allocated to suboptimal actions with reward $ = 0$, equalling a loss of $a$ compared to the best rewarding (and only rewarding) action. 
This implies that, up $K \leq \frac{A+1}{2}$ (say $A = 2K-1$) where $K$ is the total number of bandit rounds, we must have (for $A \geq 2$):
\begin{align}
	\forall \mathscr A, \quad  \E \Big[ \textsc{Regret}(\mathscr A)  \Big] \geq  \underbrace{(1-\frac{1}{2K-1})}_{\substack{\text{expected fraction of} \\ \text{non-optimal arms pulled}}} \times \underbrace{a}_{\text{loss of any suboptimal arm}} \times \underbrace{K}_{\text{\# rounds}} \geq  \frac{1}{2} \times \( \epsilon \sqrt{\frac{d-1}{8\ln(A)}} \) \times K = \Omega(\sqrt{d}\epsilon K).
	\end{align}
	
Therefore we have the proved the following proposition (notice that this is a slight generalization of \cite{du2019good}, in that we allow $\epsilon$ to be smaller than $\frac{1}{\sqrt{d}}$ and study the best achievable regret).
\begin{restatable}[Misspecified Linear Bandit Regret Lower Bound]{prop}{MissLinBandit}
	There exists a feature map $\phi : \mathcal A \rightarrow \R^d$ that defines a misspecified linear bandit class $\mathcal M$ such that every bandit instance in that class 
	has reward response:
	\begin{align}
		\mu_a = \phi_a^\top\theta + c_a
	\end{align}
	for any action $a$ (here $c_a \in [0,\epsilon]$ is the deviation from linearity and $\mu_a \in [0,1]$) and such that the expected regret of any algorithm on at least a member of the class (for $A \geq 3$) up to round $K$ is $ \Omega(\sqrt{d} \epsilon K)$.
\end{restatable}

\subsection{Lower Bound Construction}
\label{sec:LowerBoundConstruction}
In the two prior sections we recalled bandit lower bounds for estimation in noisy and misspecified linear bandits. Combining the two yields the result for our setting.

More precisely we construct a class of MDPs where each MDP comprises two parts: the noisy ``linear'' part of the MDP, denoted with $L$, that contains a one-shot bandit problem at timestep $t=1$ and no reward for later timesteps $t=2,\cdots,H$ which complies with linearity and gives the statistical lower bound; the ``misspecification'' part, denoted with $M$ which deviates from linearity by $\epsilon$ and therefore induces the misspecification lower bound. Since the starting state is arbitrary (and it can even be chosen adversarially) then alternating the starting state from the $L$ to the $M$ part of the MDP gives the result. More precisely, let there be two possible starting states $s^{M}_1$ and $s^{L}_1$, and let the starting state be $s^{M}_1$ ($s^{L}_1$, respectively) every other episode.

\subsubsection{Misspecified Chain - Rewards and Dynamics}
\label{sec:MisspecifiedChainDynamics}
   If $s^M_{1}$ is the starting state then the agents enters into the ``misspecified'' area of the MDP, made of linear bandits with a similar construction as in \cite{lattimore2020learning,du2019good}. In particular, we have the states $\{ s^M_t \}_{t=1,\dots,H}$. Any action from a generic $s^M_t$ gives a deterministic transition to the state indexed $s^M_{t+1}$, for any $t \in [H]$. There are $A$ actions in every state. The rewards upon taking action $a$ in timestep $t \in {2,\dots,H}$ is $\mu_{ta} \in [0,\frac{1}{2H}] \subset \R^A$ but is $0$ in $s^{M}_{1}$.
   
 \subsubsection{Misspecified Chain - Featurization}
\label{sec:MisspecifiedChainFeaturization} 
   The feature extractor for $t=1$ is $\phi_1(s^M_{1},\cdot) = [\overbrace{0,\cdots,0}^{\overline d},1/2]$ which has dimension $\overline d+1$; there is only one action available at the starting state. 
   For $t>2$ the feature extractor is $\phi_t(s_t^M,a) = \frac{1}{2} [\Phi[a,:],1]$, of dimension $d_t$. The construction is such that $\Phi$ is used for the reward response, and the bias is used to represent the next-state value function.
      
   Here  $\Phi$ is the matrix described in \cref{lem:JL} (i.e., with $2$-norm of the rows of value $1$). Notice that $ \forall a, \; \| \phi_t(s_t^M,a) \|_2 \leq \Lphi = 1$ satisfies our hypothesis on the feature bound.
  
\subsubsection{Linear Bandits - Rewards and Dynamics}
\label{sec:LinearBanditDynamics}
When starting in state $s^{L}_{1}$, the first step is a linear bandit problem in terms of reward response (in particular with response $\phi(s^{L}_1,a)^\top [{\thetastar}^{,L},0] + \eta$ with $1$-subGaussian noise and a unique transition to the state $s_{2}^L$. In particular, the feature $\phi(s^{L}_1,\cdot)$ has the last component equal to zero. Later states (so for $t = \{2,\cdots,H \}$) have no rewards and have deterministic transition to from $s_t^L$ to $s_{t+1}^L$. 

\subsubsection{Linear Bandits - Featurization}
\label{sec:LinearBanditFeaturization}
The features in $s_{1}^L$ have the first $\overline d$ components on a $\overline d$ dimensional hypersphere, as per the construction in Theorem 24.2 of \cite{lattimore2020bandit}  but divided by $2$, and the last component (the ``bias'') is set equal to $1/2$; the fact $\|\phi(s^{L}_1,a)\|_2 \leq \Lphi = 1$ follows. At later timesteps (i.e., $t \geq 2$) we set $\phi^L_t(s^L_{t},\cdot) = 0 \in \R^{d_t}$.

\subsubsection{Computation of Inherent Bellman Error}
\label{sec:Computationof InherentBellmanError}
Define the value function classes, for each $t \in [H]$:
\begin{align}
	\mathcal Q_t & = \Big\{ (s,a) \mapsto \phi_t(s,a)^\top \theta_t \quad \textrm{such that} \quad |\phi_t(s,a)^\top \theta_t |\leq \frac{H-t+1}{H} \Big\} \\
	\mathcal V_t & = \Big\{ (s) \mapsto \max_a \phi_t(s,a)^\top \theta_t \quad \textrm{such that} \quad |\phi_t(s,a)^\top \theta_t |\leq \frac{H-t+1}{H} \Big\} 
\end{align}
Notice that at any timestep $t \in [H]$ the only state possible is $s^M_{t}$ or $s^L_{t}$ depending on whether the starting state was $s^M_{1}$ or $s^L_{1}$, respectively. 
\paragraph{Inherent Bellman Error at Timestep $t=1$}
Notice that the model is linear at timestep $t=1$: for any $V_2 \in \mathcal V_2$ we can write:
\begin{align}
	(\T_1 V_{2})(s^L_1,a) & =  \phi_1(s^L_1,a)^\top [{\thetastar}^{,L},0]\\
	(\T_1 V_{2})(s^M_1,a) & = V_{2}(s^M_2).
\end{align}
Notice that that $\overline \theta_1 = [{\thetastar}^{,L},2V_{2}(s^L_2)]$ can precisely represent such backup:
\begin{align}
	\phi_1(s^L_1,a)^\top [{\thetastar}^{,L},2V_{2}(s^L_2)] = \phi_1(s^L_t,a)[1:\overline d]^\top {\thetastar}^{,L} + 0 * 2V_{2}(s^M_2)  & = \phi_1(s_1^L,a)^\top [{\thetastar}^{,L},0] \\
	\phi_1(s^M_1,a)^\top [{\thetastar}^{,L},2V_{2}(s^M_2)] = 0^\top {\thetastar_1}^{,L} + \frac{1}{2} * 2V_{2}(s^M_2) & = V_{2}(s^M_2).
\end{align}
Finally, notice that $\|\overline \theta_2\|^2_2 = \|[{\thetastar}^{,L} ,2V_{2}(s^L_2) ]\|^2_2 = \|{\thetastar}^{,L}\|_2^2 
+ (2V_{2}(s^L_2))^2 \leq 1+2 \leq \overline d$ since $\| {\thetastar}^{,L}\|_2 \leq \frac{\overline d^2}{48K_L} $ as the construction is the same as in \cref{prop:BanditLowerBound} (here $K_L$ is the number of episodes spent in section $L$ of the MDP). The condition $\overline d \geq 3$ will be put as assumption on \cref{thm:MasterLowerBound}.

\paragraph{Inherent Bellman Error at Timestep $t>1$}

We show that the inherent Bellman error is $\IBE = \frac{\epsilon}{2H}$ (this will be the value of the inherent Bellman error for the full MDP). For any timestep $t = 2,\dots,H$ (so excluding $t=1$) and $V_{t+1} \in \mathcal V_{t+1}$:
\begin{align}
	(\T_t V_{t+1})(s^L_t,a) & = 0 \\
	(\T_t V_{t+1})(s^M_t,a) & = \mu_{ta} + V_{t+1}(s^M_{t+1}).
\end{align}
where $\mu_{ta} \in H^{\epsilon}_{\Phi,a}$ with $a = \epsilon \sqrt{\frac{d-1}{8\ln(A)}}$. 

The feature matrix (for all the $A$ actions) is $\frac{1}{2}[\Phi,\1]$ in state $s^{M}_{t}$ and $[0,\dots,0]$ for the only action in state $s^{L}_{t}$. Using the above display, we can compute the $\overline \theta_t$ that minimizes the largest of the two following quantities (to compute a bound on $\IBE$):
\begin{align}
	\| [0,\cdots,0]^\top \overline \theta_t - \underbrace{V_{t+1}(s^L_{t+1})}_{=0} \|_\infty \\
		\| \frac{1}{2}[\Phi,\1] \overline \theta_t - \(\mu_{t} + V_{t+1}(s^M_{t+1})\1 \)\|_\infty &
\end{align}
The first is $=0$ for all choices of $\overline \theta_t$ and $\overline \theta_{t+1}$. For the second, use the triangle inequality:
\begin{align}
	\| \frac{1}{2}[\Phi,\1] \overline \theta_t - \(\mu_{t} + V_{t+1}(s^M_{t+1})\1 \)\|_\infty \leq \| \frac{1}{2}\Phi\overline \theta_t[1:d_t-1] - \mu_{t}\|_\infty + \|\frac{1}{2}\1\overline \theta_t[d_t] -  V_{t+1}(s^M_{t+1})\1 \|_\infty 
\end{align}
The second term can be made $0$ by choosing $\theta_t[d_t] = 2 V_{t+1}(s^M_{t+1}) \in [0,1]$. The first term can be made $\leq \epsilon$ (with $\epsilon$ to be defined in few steps). This implies that $\IBE = \epsilon$; here $\epsilon$ is an upper bound on the approximation error, and in particular, $\epsilon$ can be chosen to satisfy\footnote{Notice that the last component of the feature is reserved for the bias} $\frac{1}{2H} = \epsilon \sqrt{\frac{d_t-2}{8\ln(A)}}$ by choosing $a = \frac{1}{2H}$ in  \cref{prop:rescaled_lowerbound}. In other words, $\IBE \leq \frac{1}{2H} \sqrt{\frac{8\ln(A)}{d_t-2}}$.

\subsubsection{Regret Calculation}
Assume that in odd-numbered episodes the starting state is $s_1^L$ and in even-numbered episodes the starting states is $s_{1}^M$. Then \cref{prop:BanditLowerBound} ensures that the expected regret up to episode $K$ is at least $\Omega(\overline d \sqrt{K})$ (in particular we choose $\overline d = d_1 = \sum_{t=2}^H d_t$). 
At the same time, the $M$ part of the chain contains $H$ misspecified problems (which can be chosen independently) and the expected regret must be $\Omega(\frac{K}{H})$ in each of the bandit (assuming $K \leq \frac{A+1}{2}$ and $A \geq 2$) using \cref{prop:rescaled_lowerbound} with $a = \frac{1}{2H}$ and the remark on regret below such proposition. 
Since the misspecified bandits can be chosen independently, the regret up to episode $K$ on section $M$ of the MDP is $\Omega(H \times \frac{K}{H})$. Given the relation $a = \epsilon \sqrt{\frac{d_t-2}{8\ln(A)}}$ to satisfy in \cref{prop:rescaled_lowerbound} with $a = \frac{1}{2H}$, we can write the regret in section $M$ of the MDP  as $\Omega(H \times \frac{K}{H}) = \Omega(\sum_{t=2}^H \sqrt{d_t} \times \frac{K}{\sqrt{d_t}H}) = \Omega(\sum_{t=2}^H \sqrt{d_t} \times \epsilon K)$. 
However, we established $\epsilon = \IBE$, so an expected regret $ \Omega(\sum_{t=2}^H \sqrt{d_t} \times \IBE K) $ follows. Since the dimensions $d_t$'s are arbitrary, we can choose $\sum_{t=2}^H d_t = d_1 = \overline d$ for simplicity. This leads to the following theorem:
\subsubsection{Theorem Statement}
Let $M({\thetastar}^{,L},\mu_{2},\dots,\mu_{H})$ be the MDP described in  \cref{sec:LowerBoundConstruction}: this MDP is function of a certain feature representation $\phi$ as described in \cref{sec:LinearBanditFeaturization,sec:MisspecifiedChainFeaturization}, where ${\thetastar}^{,L}$ is the parameter for the linear bandit response of \cref{sec:LinearBanditDynamics} and the $\mu_{t}$'s are the reward response vectors for the misspecified subpart of the MDP as described in \cref{sec:MisspecifiedChainDynamics}. Next, consider the set $\mathcal M$ of MDPs (which depends on the horizon $H$ and misspecification $\epsilon$) defined by the MDP just explained with varying parameters:
$$
\mathcal M \defeq \{M({\thetastar}^{,L},\mu_{2},\dots,\mu_{H}) \mid \|{\thetastar}^{,L} \|_2 \leq 1, \mu_{t} \in \mathcal H^{\epsilon}_{\Phi_t,a}, t = 2,\dots,H \}
$$ 
with $a = \epsilon \sqrt{\frac{d_t-2}{8\ln(A)}}$ for any $t = 2,\dots,H$ and $\mathcal H^{\epsilon}_{\Phi_t,a}$ as described in \cref{eqn:H}.  As computed in \cref{sec:Computationof InherentBellmanError} we have that $\IBE = \epsilon$ for any MDP in the class. Notice that that every MDP in the class is defined through certain feature maps $\phi_1,\dots,\phi_H$, which are shared among all MDPs in the class. We have proved the following:
\LowerBound

\newpage
\section{Misspecified Contextual Linear Bandit}
\label{sec:Bandits}
We briefly verify \cref{cor:ConMissLinBand}. In particular, assumption \ref{ass:MainAssumption} is satisfied since the maximum return is $1$ in this setting; the features are certainly $\| \phi(\cdot,\cdot) \|_2 \leq 1$ by assumption; the rewards are $1$ sub-Gaussian by assumption and there are no transitions;
$\| \thetastar \|_2 \leq \sqrt{d}$ and finally $\mathcal B \defeq \{\theta \in \R^d \mid \|\theta \|_2 \leq \sqrt{d} \}$.
Then the optimization program that \Alg{} solves reads (after simplification and removal of the constraint $\overline \theta \in \mathcal B$):

	\begin{align*}
	& \max_{\substack{\overline \xi,\widehat \theta,\overline \theta}} \quad \max_a \phi(s_{k},a)^\top \Bigg[ \underbrace{\(\sum_{i=1}^{k-1}\phi^\top_{i}\phi_{i}^\top + \lambda I \)^{-1}\sum_{i=1}^{k-1}\phi_{i} r_{i}}_{\widehat \theta} + \overline \xi \Bigg]  \quad \textrm{subject to}  \\
	& \quad  \|\overline \xi\|_{\Sigma_{k}} \leq \sqrt{\alpha_{k}}
		\end{align*}
	It is possible to further simplify the objective, by ``aligning'' $\overline \xi$ to $\phi(s_k,a)$, obtaining:
	\begin{align*}
	& \max_{\substack{\overline \xi,\widehat \theta,\overline \theta}} \quad \max_a \Bigg[\phi(s_{k},a)^\top \(\sum_{i=1}^{k-1}\phi^\top_{i}\phi_{i}^\top + \lambda I \)^{-1}\sum_{i=1}^{k-1}\phi_{i} r_{i}  +\|\phi(s_k,a) \|_{\Sigma^{-1}_{k}}\underbrace{\|\overline \xi\|_{\Sigma_{k}}}_{\defeq \sqrt{\alpha_{k}}}\Bigg]
	\end{align*}
	which is computationally tractable (depending on the size of the action space). This coincides with the classical \textsc{LinUCB} algorithm with $\sqrt{\alpha_k} = \widetilde O(\sqrt{d})$ exploration parameter when $\IBE = 0$; otherwise, the exploration parameter becomes $\sqrt{\alpha_k} = \widetilde O(\sqrt{d} + \sqrt{k}\IBE)$. In other words, we need to add $\sqrt{k}\IBE$ to compensate for misspecification. In fact, it is possible to prove that \textsc{LinUCB} can fail in misspecified linear bandit, unless the $\sqrt{k}\IBE$ correction is made to the exploration parameter $\sqrt{\alpha_k}$. Finally, such correction partially appeared in \cite{jin2020provably,zanette2020frequentist} for a different setting, but here we use a tighter projection argument to save a $\sqrt{d}$ factor (our projection argument can be applied to their analyses as well).

\end{document}